\documentclass{article}
\pdfoutput=1
\usepackage{authblk}
\usepackage[utf8]{inputenc} 
\usepackage[T1]{fontenc}    
\usepackage{lys}
\usepackage{booktabs}       
\usepackage{nicefrac}       
\usepackage{microtype}      

\makeatletter
  \def\title@font{\Large\bfseries}
  \let\ltx@maketitle\@maketitle
  \def\@maketitle{\bgroup%
   \let\ltx@title\@title%
    \def\@title{\centering\resizebox{\linewidth}{!}{%
      \mbox{\title@font\ltx@title}%
    }}%
    \ltx@maketitle%
    \addvspace{-1\baselineskip}\egroup}
\makeatother

\usepackage[verbose=true,letterpaper]{geometry}
\AtBeginDocument{
  \newgeometry{
    textheight=9in,
    textwidth=5.5in,
    top=1in,
    headheight=12pt,
    headsep=25pt,
    footskip=30pt
  }
}
\usepackage[round]{natbib}
\title{Smooth Bandit Optimization: Generalization to H\"older Space}
\begin{document}
\date{}
\author[1]{Yusha Liu\thanks{yushal@cs.cmu.edu}}
\author[2]{Yining Wang \thanks{yining.wang@warrington.ufl.edu}}
\author[1]{Aarti Singh \thanks{aarti@cs.cmu.edu}}
\affil[1]{\normalsize Machine Learning Department, Carnegie Mellon University}
\affil[2]{\normalsize Department of Information Systems and Operations Management, University of Florida}
\maketitle
\begin{abstract}

  We consider bandit optimization of a smooth reward function, where the goal is cumulative regret minimization. This problem has been studied for $\alpha$-H\"older continuous (including Lipschitz) functions with $0<\alpha\leq 1$. Our main result is in generalization of the reward function to H\"older space with exponent $\alpha>1$ to bridge the gap between Lipschitz bandits and infinitely-differentiable models such as linear bandits. For H\"older continuous functions, approaches based on random sampling in bins of a discretized domain suffices as optimal. In contrast, we propose a class of two-layer algorithms that deploy misspecified linear/polynomial bandit algorithms in bins. We demonstrate that the proposed algorithm can exploit higher-order smoothness of the function by deriving a regret upper bound of $\tilde{\bigo}(T^\frac{d+\alpha}{d+2\alpha})$ {for when $\alpha>1$}, which matches existing lower bound. 
  We also study adaptation to unknown function smoothness over a continuous scale of H\"older spaces indexed by $\alpha$, with a bandit model selection approach applied with our proposed two-layer algorithms. We show that it achieves regret rate that matches the existing lower bound for adaptation within the $\alpha\leq 1$ subset.  
\end{abstract}
\section{Introduction}

This paper considers the problem of black-box optimization of a reward function 
$f: \Xcal\rightarrow\mathcal{R}$, that is bounded and defined on a compact $\xdim$-dimensional domain $\Xcal$, using active queries. At each round, the learner chooses an action $x_t$ by leveraging the previously collected data and observes a noisy and zeroth order feedback of the function value $f(x_t)$. In the bandit setting, the goal is to minimize the cumulative regret with respect to global maxima. This is also known as the continuum-armed bandit problem. The bandit framework is different from 
standard global zeroth order optimization 
because of its unique exploration-exploitation dilemma. 
While in zeroth order optimization problems, pure exploration will often suffice since the performance is measured by simple regret (i.e. difference between the optimized function value and true function maxima), in bandit optimization, the queried function values need to be controlled through the entire optimization process to minimize the cumulative regret. Therefore, the algorithms require different and often more careful design. 

Most existing works on continuum-armed bandit optimization either assume parametric models such as linear bandits (\cite{dani2008stochastic, abbasi2011improved, rusmevichientong2010linearly})
for the reward function, or a black-box model where the reward function is assumed to be $\alpha$-H\"older continuous (including Lipschitz) with $0<\alpha\leq 1$ with respect to some known metric~(\cite{kleinberg2005nearly, auer2007improved, kleinberg2008multi, bubeck2010x, bubeck2011lipschitz, locatelli2018adaptivity}). The main purpose of this paper is to extend this assumption to the more general H\"older function space (definition~\ref{def: Holder space}) with exponent $\alpha > 1$ and exploit the higher order of function smoothness.
{Generalization to $\alpha>1$ is a parallel to the H\"older assumpions in fundamental results in nonparametric regression~(\cite{stone1982optimal}), which has been used in a variety of applications such as economics~(\cite{yatchew1998nonparametric}).}
Approaches based on fitting an appropriate function using random samples in bins of a discretization of the domain {(i.e., exploration)} suffice as optimal for controlling cumulative regret for H\"{o}lder continuous reward functions with $\alpha\leq 1$, as well as controlling simple regret of  H\"{o}lder smooth reward functions with any $\alpha > 0$. In contrast, controlling cumulative regret for H\"{o}lder smooth reward functions with $\alpha > 1$ requires finer control in bins over the queried values via a local exploration-exploitation tradeoff. Thus, instead of using a single layer algorithm that randomly samples from selected bins, we propose a class of algorithms that use two layers of bandit algorithms - one multi-armed bandit algorithm operating over the bins, and another set of misspecified linear/polynomial bandit algorithms operating 
in each bin to govern the local exploration-exploitation tradeoff.  We derive regret bounds for this class of two-layer bandit algorithms and show that they match the existing lower bounds apart from log factors. 

Additionally, we study the problem of adaptation to smoothness exponent $\alpha$ for a continuous scale of H\"older spaces.
Unlike the simple regret minimization setting where this adaptation comes at no cost in terms of the minimax rates, it was shown by~\cite{locatelli2018adaptivity} that 
it is generally impossible to achieve minimax adaptation under cumulative regret. 
We propose a procedure with regret bound that matches the existing adaptive lower bound 
with only access to the range of the unknown parameter $\alpha$.
We start by describing related works, followed by a summary of our contributions. 




\subsection{Related Works}
\paragraph{Continuum-armed Bandit.}
In continuum-armed bandit problems, the domain $\Xcal$ is allowed to be a measurable space, and the set of arms is therefore infinite. Previous works in continuum-armed bandit usually assumes global smoothness (\cite{kleinberg2005nearly}) of the reward function or local smoothness (e.g. \cite{auer2007improved}) around the global maxima. The smoothness condition, in particular, is defined as Lipschitz continuity with respect to some metrics (\cite{kleinberg2005nearly, kleinberg2008multi}) or dissimilarity functions (\cite{kleinberg2008multi, bubeck2010x}), or $\alpha$-H\"older continuity with $0<\alpha\leq 1$ (\cite{kleinberg2005nearly,auer2007improved}). Worst-case lower bound under the Lipschitz assumption is presented in~\cite{kleinberg2008multi} and that under the H\"older continuity assumption in~\cite{locatelli2018adaptivity}. 

Existing works 
rarely consider the generalization to H\"older space. 
Recently ~\cite{hu2020smooth} studied contextual bandit with reward functions in H\"older spaces,
however, the reward function is assumed to be smooth with respect to the observed contexts and the action set is finite. For non-contextual continuum-armed bandits, \cite{akhavan2020exploiting} focus on the strongly convex subset of 
functions in H\"older spaces with $\alpha\geq 2$ by using projected gradient-like algorithms. \cite{grant2020thompson} analyze Thompson sampling (TS), a Bayesian method, on H\"older spaces with integer-valued exponents and derive a suboptimal upper bound based on the complexity of the function space\footnote{They comment that the reason could be either the analysis being suboptimal or the nature of TS. They also derive lower bounds under one-dimension setting, but as we later remark in this paper, the same lower bound can be implied by \cite{wang2018optimization} under a more general setting. }.

\paragraph{Adaptivity to Smoothness of the Reward Functions.}
An intriguing problem is whether an algorithm that is oblivious to the H\"older exponent $\alpha$ can simultaneously achieve minimax rates for a range of values for $\alpha$. For non-contextual continuum-armed bandits, this has been discussed only under the H\"older continuous($\alpha\leq 1$) setting.
In particular, ~\cite{locatelli2018adaptivity} 
state that generally, such minimax adaptation to $\alpha$ is impossible by providing a worst-case lower bound for adaptation between two H\"older-continuous function spaces. Additionally, they propose conditions under which it would become possible. (For the contextual finite-armed bandit studied in~\cite{hu2020smooth}, \cite{gur2019smoothness} provide lower bounds with similar rates and the extra conditions as well.) However, it remains unclear that, without the extra conditions, whether an algorithm can achieve the lower bound when adapting to a continuous scale of general H\"older spaces.   

\paragraph{Model Selection for Bandits.}
Another relevant line of work is more broadly model selection in bandit settings, which we will leverage in bandit optimization of H\"older-smooth functions as well as adaptation to the smoothness. In this problem, given a set of base algorithms on possibly different domains, the learner needs to adapt to the best one in an online fashion. The goal is to achieve cumulative regret comparable to the best base algorithm if it were run solely. \cite{bubeck2011lipschitz} study the model selection problem for adapting to the unknown Lipschitz constant of functions. 
\cite{foster2019model} study adapting to the unknown policy dimension in contextual linear bandits by estimating the gap between two policy classes. \cite{agarwal2016corralling} develop a general algorithm named Corral for bandit model selection under adversarial feedback. It uses online mirror descent to balance between base algorithms. For stochastic feedback particularly, \cite{pacchiano2020model} modify the Corral algorithm to relax requirements on base-algorithms and improve the result on some problem instances (including the one in~\cite{foster2019model}). Another relevant issue addressed in~\cite{krishnamurthy2019contextual} which study contextual continuum-armed bandits with Lipschitz continuous reward functions, is their use of the original Corral algorithm applied with EXP4 for adaptation to unknown Lipschitz constant. UCB-type algorithm for corralling base-algorithms is used in~\cite{arora2020corralling} under the assumption that the base-algorithms are finite-armed, and only one of them has access to the best arm. 

\subsection{Our Contributions}
We study bandit optimization of functions in general H\"older spaces. This paper furthers the previous works in the following two main aspects:
\begin{enumerate}
    \item We propose a novel class of two-layer bandit algorithms, where a carefully-chosen Meta-algorithm deploys misspecified bandit algorithms as arms. Our algorithms show explicitly how to exploit higher-order smoothness in achieving optimal exploration-exploitation tradeoff. We 
    derive worst-case regret bound for this algorithm that matches the existing lower bound except for log factors, for functions in H\"older space including {when $\alpha> 1$}. Our results bridges the gap between Lipschitz smooth bandits where the H\"older exponent is $\alpha=1$ and infinitely-differentiable problems such as linear bandits where the H\"older exponent is $\alpha=\infty$. 
    \item We study adaptation to a sequence of H\"older spaces indexed by a continuous but unknown variable of exponent $\alpha$. We propose a strategy with theoretical guarantee, which uses the bandit model selection algorithm Corral from~\cite{pacchiano2020model} applied with versions of our proposed two-layer algorithms. 
    The derived regret bound is to our knowledge the first result on upper bounds when adapting to a continuous scale of H\"older spaces in continuum-armed bandit optimization. 
\end{enumerate}

The rest of this paper is organized as follows: 
In section 2 we introduce the problem formulation and assumptions. We present the two-layer Meta-algorithms and theoretical guarantees in section 3. In section 4 we study the adaptation to unknown smoothness and conclude the paper in section 5 with some open questions. 
\section{Problem Formulation}
In this paper, we consider bandit optimization of smooth functions in H\"older space $\sum(\alpha, L)$ with $\alpha>1$. 
The H\"older space is defined formally in definition~\ref{def: Holder space}. 
Some works also study benign problem instances with additional ``growth" conditions  than the smoothness to characterize the difficulty of finding global maxima, for improvements in regret bounds. For example, \cite{auer2007improved} use a parameter to model the growth rate of Lebesgue measure of the near-optimal arms set as a function of the threshold. The near-optimality dimension in~\cite{bubeck2010x} uses packing number but has similar meaning. 
In this paper we will focus solely on worst-case regret to preserve simplicity and leave adaptation to benign cases as a future direction. The performance of the learner is measured by cumulative pseudo-regret as stated below where $x^* \in \argmax_{x\in\Xcal} f(x)$. 
Throughout this paper we will simply refer to the pseudo-regret as regret. 
\begin{align}\label{eq: def of pseudo-regret of function}
    R(T) &=\sum_{t=1}^T[f(x^*) - f(x_t)].
\end{align}
To formally define H\"older spaces, we first introduce some notations. Define the following notions for a vector $s=(s_1\dots s_\xdim)$: let $\vert s\vert = s_1 + \dots + s_\xdim$, $s! = s_1!\dots s_\xdim$ and $x^s = x_1^{s_1}\dots x_\xdim^{s_\xdim}$. And define $D^s = \frac{\partial^{\vert s\vert}}{\partial x_1^{s_1}\dots \partial x_\xdim^{s_\xdim}}$.
\begin{definition}[\cite{tsybakov2008introduction}]\label{def: Holder space}
    The H\"older space $\sum(\alpha, L)$ on domain $\Xcal\in \mathcal{R}^\xdim$ is defined as the set of functions $f:\Xcal\rightarrow \mathbf{R}$ that are $l=\floor{\alpha}$ times differentiable and have continuous derivatives\footnote{Only when referring to the order of H\"older smooth functions' derivatives do we denote $\floor{\cdot}$ as the largest integer \textit{strictly} less than input. In other places in this paper it denotes less or equal to input. }. $l$ is the largest integer that is strictly smaller than $\alpha$. A function $f$ in $\sum(\alpha, L)$ satisfies the following inequality\footnote{We use $l_\infty$ norm as in some works on adaptive confidence bands and optimization (\cite{low1997nonparametric,tsybakov2008introduction,hoffmann2011adaptive, wang2018optimization}).} for $\forall x, y\in\Xcal$. 
    \[D^sf(x) - D^sf(y) \leq L\lVert x-y\rVert_\infty^{\alpha-l},\quad  \forall s \  \text{s.t.} \vert s\vert=l .\]
    In particular, a function in $\sum(\alpha, L)$ is close to its Taylor approximation:
    $$\lvert f(x) - T_y^{l}(x)\rvert\leq L\lVert x-y\rVert_\infty^\alpha, \forall x, y \in \Xcal.$$ 
    We use $T_y^{l}$ to denote the $l$-degree Taylor polynomial around $y$,
        $T_y^{l}(x) = \sum_{\vert s\vert\leq l}\frac{(x-y)^s}{s!}D^sf(y).$
\end{definition}
\paragraph{Assumptions}
We specify the assumptions that are used throughout this paper. 
\begin{description}
\item[G1.]\label{assumption: global value bound}
The input domain $\Xcal$ is a hypercube $[0,1]^\xdim$. 
For simplicity assume the reward function is bounded: $\Vert f\Vert_\infty\leq 1$.  
\item[G2.]\label{assumption: holder smooth}
The function $f$ belongs to H\"older space $\sum(\alpha, L)$ with some constant $L>0$ \footnote{In this paper, for simplicity, we assume $L$ is some constant that satisfies assumption G1.}. 
\item[G3.]\label{assumption: bounded noise}
The observations are noisy: $y = f(x) + \eta$ where the noise $\eta$ is drawn from i.i.d zero mean sub-gaussian distribution with parameter~$\sigma$.
\end{description}
\section{Meta-algorithm and Analysis}\label{sec: meta algorithm}

A commonly used method for continuum-armed bandits is fixed discretization, which divides the continuous input domain into finite number of bins, to transform the problem into finite-armed bandit. 
Previous works mostly consider H\"older-continuous ($\alpha\leq 1$) functions. For example~\cite{auer2007improved} study the $\alpha$-H\"older continuous functions with $\alpha\leq 1$ 
for one-dimension domain, 
followed by~\cite{bubeck2010x} who generalize it to $\xdim$-dimensional domain and propose the HOO algorithm with adaptive discretization\footnote{The adaptive discretization does not change worst-case regret but has improvements on benign problems, as introduced in section 2.}.
In these works, it suffices to perform random sampling (\cite{auer2007improved, bubeck2010x}) or midpoint sampling (\cite{kleinberg2005nearly}) inside each bin. 
The worst-case regret bound for Lipschitz space of $\tilde\bigo(T^\frac{\xdim+1}{\xdim+2})$ are matched by the general lower bound of $\Omega(T^{\frac{\xdim+\alpha}{\xdim+2\alpha}})$(\cite{auer2007improved, bubeck2010x, locatelli2018adaptivity, bubeck2011lipschitz})
apart from log factors.  
However, if we apply the same methods of random sampling on fixed discretization (\cite{auer2007improved}) on functions with H\"older exponent $\alpha> 1$,
the regret incurred is $\tilde{\bigo}(T^\frac{\xdim+1}{\xdim+2})$ since the H\"older space with exponent $\alpha> 1$ is a subset of the Lipschitz function space. It prompts us to ask the question of whether a better rate that matches the dependence on $\alpha$ can be achieved for functions that are smoother than Lipschitz. 
An extreme is when $\alpha$ reaches infinity, where the reward model will be infinitely-differentiable, for example the stochastic linear bandit which enjoys $\tilde{\bigo}(T^\frac{1}{2})$ regret even on continuous domain (\cite{dani2008stochastic,abbasi2011improved}).  

\subsection{Algorithm Overview}
We keep to fixed discretization of the domain since we consider only the worst-case regret. We divide $\Xcal = [0,1]^\xdim$ into $n$ equal-sized hypercubes, leaving $n$ as a parameter of the algorithm.
As shown in definition~\ref{def: Holder space}, the function is locally well-approximated by Taylor polynomial which reduces to a {linear model of a feature map of $x$ with dimension $\lineardim$}. It is equivalent to observing a misspecified linear model inside each bin, the equivalence formally quantified in Lemma~\ref{lemma: local bias}. Therefore, local exploration-exploitation tradeoff can be achieved by a base algorithm with sublinear regret on such misspecified models, with a Meta-algorithm to balance the budgets between the base algorithms in the bins. 
\begin{lemma}\label{lemma: local bias}
    Let hypercube $\mathcal{B}_{\Delta}$ be a subset of the input space with volume $\Delta$. If a function satisfies assumption $G1\sim2$, there exists a linear parameter\footnote{We slightly abuse the notation and define short-hand notation $\la \theta,x\ra:= \theta_0+\sum_{i=1}^{\lineardim} \theta_i x_i$.} $\theta^* \in R^{\lineardim}$ and {feature map $\phi: x\mapsto \phi(x)\in R^{\lineardim}$, 
    such that $f$ can be approximated by the linear function: $\lVert f- \la\theta^*, \phi(x)\ra \rVert_\infty \leq \epsilon = {L} \Delta^{\frac{\alpha}{d}}$ for $x\in B_\Delta$. When $\alpha\leq 2$, $\lineardim = \xdim$; when $\alpha>2$, $\lineardim = \bigo(\xdim^l)$ with $l$ (definition~\ref{def: Holder space}).} 
    {Note that the linear parameter may not be unique. }
\end{lemma}
The proof is in Appendix section~\ref{proof of lemma local bias}. 
In the following parts of this section we first present the misspecified bandit algorithm to run inside a bin, and then the Meta-algorithms to control these local algorithms.
\subsection{The Misspecified Linear Bandit Algorithm}\label{sec: misspecified linear bandit}
In this subsection we escape from the big picture briefly in order to present the misspecified linear bandit algorithm, modified from the \textit{ConfidenceBall}$_2$ algorithm in~\cite{dani2008stochastic} to serve as ``arms'' of the Meta algorithm. The algorithm, as shown in its name, is based on construction of confidence ellipsoid of the unobserved linear parameter {in dimension $\xdim$}. 
We prove that the proposed modification can accommodate bias in the function feedback by deriving
an upper bound on the cumulative regret\footnote{For clarity this use of $\tilde\bigo$ omits $\ln(T)$ and $\delta$ dependence.} of $\tilde{\bigo}({d}\sqrt{T}+{d}T\epsilon )$. Here $\epsilon$ is the upper bound on bias value and known by the algorithm. 
We recently discovered that a  similar result with proof sketch already appeared in recent work of~\cite{lattimore2019learning} (appendix E) who used modification of the algorithm in~\cite{abbasi2011improved}, and hence enjoys the improvement of a multiplicative factor $\sqrt{\log(T)}$. For completeness and to provide necessary intermediate results for Meta-algorithms in later sections, we present our algorithm and full proof as complementary. It is worth mentioning that without the modification, the original algorithm incurs suboptimal regret under misspecification. 
\paragraph{Assumptions} We make the following assumptions for the misspecified model. Note that they are consistent with the aforementioned global assumptions. 
\begin{description}
\item[A1.]\label{A1: model}
The feedback model is $y = \la x, \theta^* \ra + b(x) + \eta$ with $\vert b(x) \vert\leq \epsilon, \forall x\in \Xcal \in {R^{\xdim}}$. 
\item[A2.]\label{A2: mean reward value}
The mean reward $\EE[y]$ is bounded by $[-1,1]$. 
\item[A3.]\label{A3: bounded noise}
The noise $\eta$ is drawn from zero-mean sub-gaussian with parameter $\sigma$\footnote{Different from~\cite{dani2008stochastic} who assumes bounded noise. This reflects in the difference in $\beta_t$.}. 
\end{description}

The pseudo-code of the modified algorithm is shown in Algorithm~\ref{alg: misspecified linear UCB}. The goal is to minimize the cumulative pseudo-regret of the linear model:
\begin{equation}
R(T) = \sum_{t=1}^T r_t =\sum_{t=1}^T(\la x^* ,\theta^*\ra - \la x_t ,\theta^*\ra).
\end{equation}
We prove that this regret is $\bigo\left(\xdim\ln(T)\sqrt{{\ln(\frac{T^2}{\delta})} T} + \epsilon T d\sqrt{2\ln(T)}\right)$ with probability $1-\delta$. This is formally stated in Theorem~\ref{thm: main regret}. 

\begin{algorithm}[h]
    \caption{Misspecified linear UCB algorithm ($\Acal^{local}$)}
    \begin{algorithmic}[1]\label{alg: misspecified linear UCB}
        \REQUIRE Misspecification error $\epsilon$, {input domain $\xdomain$ and its dimension $\xdim$}, fail probability $\delta$, upper bound on $\norm{x}^2$: $\kappa^2=d$. 
        \STATE Initialize $A_1 = I_d$ and $x_1 \in {\xdomain}$ randomly. 
        \FOR{$t = 1\dots $}
            \STATE Execute action $x_t$ and observe reward $y_t$ 
            \STATE $A_{t+1} = A_t + x_t x_t^T$
            \STATE $\hat\theta_{t+1} = A_{t+1}^{-1}(\sum_{\tau=1}^t y_\tau x_\tau)$
            \STATE $\beta_{t+1} = 128\sigma^2d\ln(1+ t)\ln(\frac{4(t+1)^2}{\delta})$
            \STATE Define function $UCB_{t+1}(x)=\left(\la x, \hat\theta_{t+1} \ra + \sqrt{\beta_{t+1}} \lVert A_{t+1}^{-1/2} x \rVert + \epsilon \sum_{s=1}^{t} \lvert x^TA_{t+1}^{-1}x_s \rvert\right)$
            \STATE Compute action $x_{t+1} = \argmax_{x\in \Xcal} UCB_{t+1}(x)$ 
            \STATE Return $x_{t+1}$ and $UCB_{t+1}(x_{t+1})$
        \ENDFOR
    \end{algorithmic}
\end{algorithm}

\begin{theorem}\label{thm: main regret}
   If assumptions A1$\sim$A3 hold, then with probability $1-\delta$, the cumulative regret of Algorithm~\ref{alg: misspecified linear UCB} is upper bounded by:
    \begin{align}
        R(T) &\leq \sqrt{8d\beta_T T\ln(1 + T)} + 2\epsilon T d \sqrt{2\ln(1 + T)} + 2\epsilon T. 
    \end{align}
\end{theorem}
The first term is the standard stochastic linear bandit regret rate same as in~\cite{dani2008stochastic}. We defer the proof to Appendix section~\ref{sec: proof of thm: main regret}. The increment of a multiplicative factor $\sqrt{d}$ in the second term compared to that in~\cite{lattimore2019learning} is due to difference in assumption on $\lVert x\rVert^2$. Their assumption is $\lVert x\rVert^2\leq 1$ whereas ours is $\lVert x\rVert^2\leq d$.

\subsection{The UCB-Meta-algorithm}
We now present the first structure of our Meta-algorithms. We consider the most straightforward structure: UCB-Meta, the pseudo-code is shown in Algorithm~\ref{alg: meta algorithm} (define $\lfloor\cdot\rceil$ as the action of rounding to nearest integer). {We keep a version of the base mispecified linear bandit algorithm in each bin. The confidence estimates of the local linear models are passed to the Meta-algorithm as UCB of arms, with adjustment of $\epsilon$, the bias quantity. At round we choose the bin with the highest UCB and run one step of the local bandit algorithm to update its estimation.} {For adjusting to different values of $l = \floor{\alpha}$, we need only to change the space that the linear model is in, specifically the feature mapping $\phi: x\mapsto \phi(x) \in R^{\lineardim}$ as defined in proof of Lemma~\ref{lemma: local bias}. For example, when $\alpha\leq 2$, the sub-algorithms are misspecified linear bandits whose actions spaces are simply bins $B\in \xdomain$. } 
\begin{algorithm}[h]
    \caption{UCB-Meta-algorithm ($\Acal^{global}$) }
    \begin{algorithmic}[1]\label{alg: meta algorithm}
        \REQUIRE smoothness parameter $\alpha$, {H\"older constant $L$},  dimension of domain $\xdim$, time horizon $T$ and fail probability $\delta$, action space $\xdomain$.
        \STATE Initialize $n = \lfloor T^{\frac{\xdim}{\xdim+2\alpha}}/\ln(T)^{\frac{2\xdim}{\xdim+2\alpha}}\rceil$ and divide the action space $\xdomain$ into same-sized bins $B_{1\dots n}$ with volume $\Delta=1/n$. 
        \FOR{$k=1, \dots,n$}
        \STATE On bin $B_k$, start a version of local misspecified base-algorithm $\Acal_k$ using misspecification error $\epsilon =  Ln^{\frac{-\alpha}{\xdim}}$, {input domain $\lineardomain=\{\phi(x), x \in \xdomain\}$ and its dimension $\lineardim$}, fail probability $\delta/n$.
        \STATE Initialize counter $s_k = 1$ to indicate how many times $\Acal_k$ is queried. 
        \STATE Query $\Acal_k$ once by running steps 3-9 of Algorithm~\ref{alg: misspecified linear UCB} with $t=s_k$ and obtain upper confidence bound $UCB_k$. 
        \STATE $s_k \leftarrow s_k +1$
        \ENDFOR
        \FOR{$\tau = 1\dots T$} 
            \STATE Select the bin with index $k(\tau) = \argmax_{k} UCB_k$.
            \STATE Execute the local bandit algorithm $\Acal_{k(\tau)}$ once by running steps 3-9 (of Algorithm~\ref{alg: misspecified linear UCB}) with $t=s_{k(\tau)}$
            \STATE Receive updated recommendation {$\phi_\tau \in \{\phi(x), x\in B_{k(\tau)}\}$} and $UCB_{k(\tau)}$.
            \STATE Advance counter for $\Acal_{k(\tau)}$: $s_{k(\tau)} \leftarrow s_{k(\tau)}+1$.
        \ENDFOR
    \end{algorithmic}
\end{algorithm}
\subsubsection{Regret Analysis of Algorithm~\ref{alg: meta algorithm}}\label{sec: regret of meta algorithm}
\begin{theorem}\label{thm: meta main regret}
    Let $\lineardim$ be the dimension of polynomial of $x$, as defined in Lemma~\ref{lemma: local bias}. If the reward function satisfies G1$\sim$G3 in section~\ref{assumption: holder smooth}, then with probability $1-\delta$, the cumulative regret (equation~\ref{eq: def of pseudo-regret of function}) of UCB-Meta-Algorithm is upper bounded by\footnote{The $d$-dependence of the second term is propagated from Theorem~\ref{thm: main regret}}
    \begin{equation}\label{eq: meta regret}
        R(T)\leq \bigo\left(\lineardim \ln(T)\sqrt{Tn\ln(T^2 n/\delta)} + 
        \lineardim\epsilon T\sqrt{\ln(T)} \right).
    \end{equation}
\end{theorem}

The core of the proof is the distribution-independent analysis of UCB, which relies on the honesty of the confidence bands as well as their lengths. In particular, if the function value $f(x)$ at time $t$ is contained in an honest confidence band $[UCB_t(x)-2l_t(x), UCB_t(x)]$, then we can use the length $l_t(x)$ to bound instantaneous regret incurred by the selected action at this step. The confidence ellipsoids for the piecewise linear parameters $\hat\theta_{k, t}$ that are constructed by local misspecified linear bandits offer a convenient confidence estimation of function value, with the additional adjustment factor $\epsilon$, the approximation error. The full proof is deferred to Appendix section~\ref{sec: proof of thm: meta main regret}. 
The algorithm defines each bin to be a hypercube with volumn $\Delta=1/n$, according to Lemma~\ref{lemma: local bias} we have $\epsilon = Ln^\frac{-\alpha}{d}$. Therefore, setting $n =\bigo( T^{\frac{\xdim}{\xdim+2\alpha}}/\ln(T)^{\frac{2\xdim}{d+2\alpha}})$ will minimize the upper bound and yield cumulative regret bound of \footnote{$\delta$-dependence absorbed in $\bigo$ since they are inside $\log$ terms.}
\begin{equation}\label{eq: final rate for UCB-meta}
{R(T)\leq\tilde{\bigo}({\lineardim}T^\frac{\xdim+\alpha}{\xdim+2\alpha}).}
\end{equation}

\subsubsection{Anytime Regret Guarantee for Algorithm~\ref{alg: meta algorithm}}\label{sec: doubling for meta-algorithm}
To achieve the rate in bound~\ref{eq: final rate for UCB-meta}, Algorithm~\ref{alg: meta algorithm} needs to know the time horizon $T$ in advance to set $n$ and $\epsilon$ correspondingly. Here we prove that, with the doubling trick (\cite{auer1995gambling}) , the UCB-Meta-algorithm can get regret that is of the same rate as in bound~\ref{eq: final rate for UCB-meta} up to constant factors without knowing $T$. This result is needed in the adaptation problem studied in section~\ref{sec: adaptivity to smoothness}. 
\begin{theorem}\label{thm: doubling for meta-algorithm}
    If 
    Algorithm~\ref{alg: meta algorithm} with access to the time horizon $T$ achieves regret of $\tilde{\bigo}(T^a)$ with probability $1-\delta$, then the procedure described in Algorithm~\ref{alg: doubling trick for meta} can achieve regret rate $\tilde{\bigo}(T^a)$ with probability $1-\delta$ without the knowledge of $T$. 
\end{theorem}
The pseudo-code for Algorithm~\ref{alg: doubling trick for meta} is in Appendix section~\ref{sec: doubling for meta algorithm} and the proof of Theorem~\ref{thm: doubling for meta-algorithm} in Appendix~\ref{sec: proof of doubling thm}.

\subsection{The Corral-Meta-algorithm}
Another choice for Meta-algorithm is bandit model selection methods. Here we use the Corral algorithm defined in~\cite{pacchiano2020model}, which will be introduced more formally in section 4. An example of corralling misspecified linear bandit algorithms without corruption to the regret rate apart from log factors has already been given in~\cite{pacchiano2020model}, but for adaptation to the misspecification error $\epsilon$. Here we demonstrate that it can also be used to corral misspecified bandit base-algorithms on different bins in a  discretized domain. We derive the following regret bound that is the same as UCB-Meta-algorithm.
\begin{theorem}\label{thm: regret of smoothcorral-meta}
    First perform the smoothing transformation (Algorithm~3 in~\cite{pacchiano2020model}) to our misspecified linear bandits in Algorithm~\ref{alg: misspecified linear UCB}, denote the smoothed misspecified linear bandits as $\Acal_s^{local}$. Then, the Meta-algorithm (Algorithm~5 (Corral-Update) reproduced in~\cite{pacchiano2020model}) applied with a set of $\Acal_s^{local}$ that are initialized in the same way as in Algorithm~\ref{alg: meta algorithm} has expected regret upper bounded by:
    \begin{equation}
        \EE[R(T)]\leq\tilde\bigo(\lineardim T^{\frac{\xdim+\alpha}{\xdim+2\alpha}}).
    \end{equation}
\end{theorem}
The proof of this theorem is in Appendix section~\ref{sec: proof of thm: regret of smoothcorral-meta}. 

\subsection{Discussion} 
The role of the Meta-algorithm is essentially model selection and adaptation to the base-algorithms. It is not a trivial task since the rewards incurred by the base-algrotihms are not i.i.d as in standard stochastic settings. However, UCB as a stochastic multi-armed bandit algorithm, is applicable as Meta-algorithm because the local parametric (linear) function approximations provide honest upper confidence bounds for each bin even under the misspecifications, thus enabling the distribution-independent analysis for UCB. 
The advantage of Corral-Meta is that it potentially allows relaxation of the H\"older smoothness to hold only around the global maxima (\cite{auer2007improved, bubeck2010x}), while the same relaxation is not straightforward for UCB-Meta. The advantage of UCB is that under standard stochastic settings where each arm has i.i.d rewards, it achieves the gap-dependent bound of $\bigo(\log(T)/\Delta)$. 
Thus an interesting question for the future is whether similar gap-dependent bounds for the UCB-Meta is available. Such bounds would enable exploitation of the growth conditions (section 2) for potential rate improvements.

\subsection{Comparison with Existing Lower Bound}
We compare the derived upper bounds of $\tilde\bigo(\lineardim T^\frac{d+\alpha}{d+2\alpha})$ to the existing lower bound from~\cite{wang2018optimization}, which study global optimization. In their work, the performance of optimization algorithms with output $\hat x_T$ is measured by simple regret 
$\mathcal{L}( \hat x_T; f) \stackrel{\triangle}{=} f(x^*) - f(\hat x_T)$, for $f$ in H\"older spaces including $\alpha\geq 1$. 
Theorem 2 (coupled with Proposition 3) in~\cite{wang2018optimization} implies that $ \sup_{f\in \sum(\alpha)}\EE[\mathcal{L}(\hat x_T; f)] = \Omega(T^\frac{-\alpha}{2\alpha+\xdim})$.
We argue that this lower bound can be directly used to lower bound the worst-case cumulative regret, by making the following observation (remark 3 in~\cite{bubeck2010x}): If a strategy achieves expected cumulative regret $\EE[R_T]$, then by uniformly selecting a past action as the final output $\hat x_T$, it can also achieve expected simple regret $\EE[\mathcal{L}(\hat x_T; f)] = \EE[R_T]/T$. 
Therefore, any strategy with cumulative regret $\tilde{o}(T \EE[\mathcal{L}(\hat x_T; f)])$ will violate the lower bound. Through proof by contradiction, we take the result from~\cite{wang2018optimization} as an $\Omega(T^\frac{\xdim+\alpha}{\xdim+2\alpha})$ lower bound on expected cumulative regret, and argue that our results match this bound up to log factors. Our results show that proposed algorithms are minimax optimal {in dependence of $T$} and effectively exploit the function smoothness. 
\section{Adaptation to Unknown Smoothness}\label{sec: adaptivity to smoothness}
In this section, we study adaptation to the smoothness exponent $\alpha$ of the reward function. Minimax adaptation, which means a learner can simultaneously achieve the minimax optimal rates (\cite{hoffmann2011adaptive,locatelli2018adaptivity}) under a nested set of H\"older spaces, has been proven to be impossible for cumulative regret minimization without additional assumptions. \cite{locatelli2018adaptivity} provide a lower bound for adaptation between two H\"older continuous functions spaces.
Assume $\alpha<\gamma\leq 1$, for any strategy with a good expected regret $\EE[R_\gamma(T)$ in $\sum(\gamma, L)]$, they show that its expected regret in the superset $\sum(\alpha, L)$ will depend inversely on $\EE[R_\gamma(T)]$, and therefore be suboptimal for $\sum(\alpha, L)$. They propose a strategy to match that lower bound that requires values of $\alpha$ and $\gamma$, thereby also proving that the lower bound is tight. 

However, when adapting to a continuous scale of H\"older spaces (possibly $\alpha\geq 1$), it remains unclear what strategy can generalize and achieve this lower bound for some H\"older spaces. We aim to answer that question by proposing a new strategy that uses a recently developed bandit model selection algorithm (Corral with smooth wrapper in~\cite{pacchiano2020model}) applied with a set of Meta-algorithms (section~\ref{sec: meta algorithm}). We will present this strategy and its theoretical guarantees next. Throughout the following sections, we refer to minimax optimal in dependence of T as minimax unless otherwise~specified. 

\subsection{Corral Applied with Meta-algorithms}
The bandit model selection method Corral is first developed by~\cite{agarwal2016corralling} and based on an instance of online mirror descent with mirror map derived from~\cite{foster2016learning}. Corral with smooth wrapper proposed by~\cite{pacchiano2020model} for stochastic feedback problems is different from the original Corral algorithm in the following aspects. The smoothed version no longer needs to send importance-weighted feedback to base-algorithm, therefore no longer requires the base-algorithms themselves to be modified for stability guarantee (definition 3 in~\cite{agarwal2016corralling}). In the following parts, we will use Corral with smooth wrapper to adapt to the smoothness and refer to it as Corral for simplicity\footnote{Since the core of oneline mirror descent in Corral is not changed.}. {A copy of the pseudo-code of Corral from~\cite{pacchiano2020model} can be found in Appendix~\ref{sec: corral algo reference} for easier reference}. 
We use a set of $M$ Meta-algorithms $\Acal^{global}(\alpha_i), i\in [M]$ in Algorithm~\ref{alg: meta algorithm} as bases. The input values $\alpha_i$ are from a grid $\mathcal{G}$ defined later. 
Therefore, we first specify the regret of a Meta-algorithm with input smoothness parameter $\alpha'$ that is ran on functions with actual H\"older smoothness $\alpha$.
\begin{lemma}\label{lemma: regret of meta with gamma}
For function $f$ that satisfies global assumptions $G1\sim G3$ with parameter $\alpha$, the regret of Algorithm~\ref{alg: meta algorithm} with input parameter $\alpha'\leq \alpha$ is bounded with probability $1-\delta$ by 
\begin{equation}
R(T)\leq \tilde\bigo(d(\alpha') T^{\frac{\xdim+\alpha'}{\xdim+ 2\alpha'}}).
\end{equation}
The bound does not hold for $\alpha' > \alpha$. 
\end{lemma}
The proof is deferred to Appendix section~\ref{proof of lemma meta regret with input gamma}.
Having established the performance of base algorithms with misspecified smoothness exponents, we present the adaptation strategy and its regret bound in Theorem~\ref{thm: adaptation, smooth corral}. Since it is impossible to achieve minimax optimal rates for multiple values of the smoothness parameter simultaneously, we introduce a user-sepecified parameter $R$ that controls the H\"older space over which minimax optimality is desired. We show that conditioned on achieving minimax rate for the space $\sum(R, L)$, our adaptation strategy provides best possible regret bound on all supersets $\sum(\alpha, L)$ where $\alpha \leq R$. The results are stated in Theorem~\ref{thm: adaptation, smooth corral}.
\begin{theorem}\label{thm: adaptation, smooth corral}
    Consider adapting to a continuous scale of nested H\"older spaces indexed by {$\alpha$ whose value is bounded in a given interval, for simplicity we assume $0<\alpha\leq 2$, where $\lineardim = \xdim$}. Define 
    $R\leq 2$ as a parameter set by the decision-maker that specifies the index of H\"older space for which minimax optimal regret is achieved. Define linear grid $\mathcal{G} = \{ \alpha_i = \frac{R}{\lfloor\log(T)\rfloor}i, i = 0, 1\dots \lfloor\log(T)\rfloor \}$ 
    so that the total number of base algorithms is $M = \vert \mathcal{G}\vert = \lceil\log(T)\rceil$. Consider using Corral with bases that are Meta-algorithms~(algorithm~\ref{alg: doubling trick for meta} in Appendix section~\ref{sec: doubling for meta algorithm}) with 
    input $\alpha_i\in \mathcal{G}, i\in [M]$. Then by setting the learning rate of Corral to be $\eta = \xdim^{-1} T^{-\frac{\xdim+R}{\xdim+2R}}$,
    the regret rates achieved for any H\"older exponent $\alpha\in (0, 2]$ are:
    \begin{align}
        &\sup_{f\in \sum(\alpha, L)}\EE[R(T)]\leq \tilde\bigo(\xdim T^\frac{d^2+2R\xdim + R\alpha}{(\xdim +2R)(\xdim+\alpha)}) \text{ for } \alpha\in (0,R], \label{eq: adapt rate smaller than R}\\
        &\sup_{f\in \sum(\alpha, L)}\EE[R(T)]\leq \tilde\bigo(\xdim T^\frac{\xdim+R}{\xdim+2R}) \text{ for } \alpha\in [R,2]. \label{eq: adapt rate larger than R}
    \end{align} 
\end{theorem}
A straightforward example is shown in Figure~\ref{fig: adaptation}. Functions with H\"older exponent $\alpha> R$ essentially belongs to a subset of $\sum(R, L)$ and have the same regret rates as in equation~(\ref{eq: adapt rate larger than R}) because the algorithm did not fully exploit their smoothness. 
There are two sources of cost of adaptation, first the cost of adapting to $M$ grid points. Since $M=\bigo(\log(T))$, this has the same difficulty as the adaptation to two values in~\cite{locatelli2018adaptivity}. The second one, however, is a consequence of adapting to a continuous scale of $\alpha$. The cost is the rate difference between the exponent $\alpha$ and the closest value to it on $\mathcal{G}$, denoted $\estalpha\in \mathcal{G}$, s.t.  $\estalpha \leq \alpha \leq \estalpha + \frac{R}{\floor{\log(T)}}$. This cost can be alleviated by the design of the linear grid. We defer the full proof to Appendix section~\ref{sec: proof of thm: adaptation, smooth corral}. 

\begin{figure}[!htbh]
    \begin{minipage}{0.48\textwidth}
      \centering
      \includegraphics[width=\textwidth]{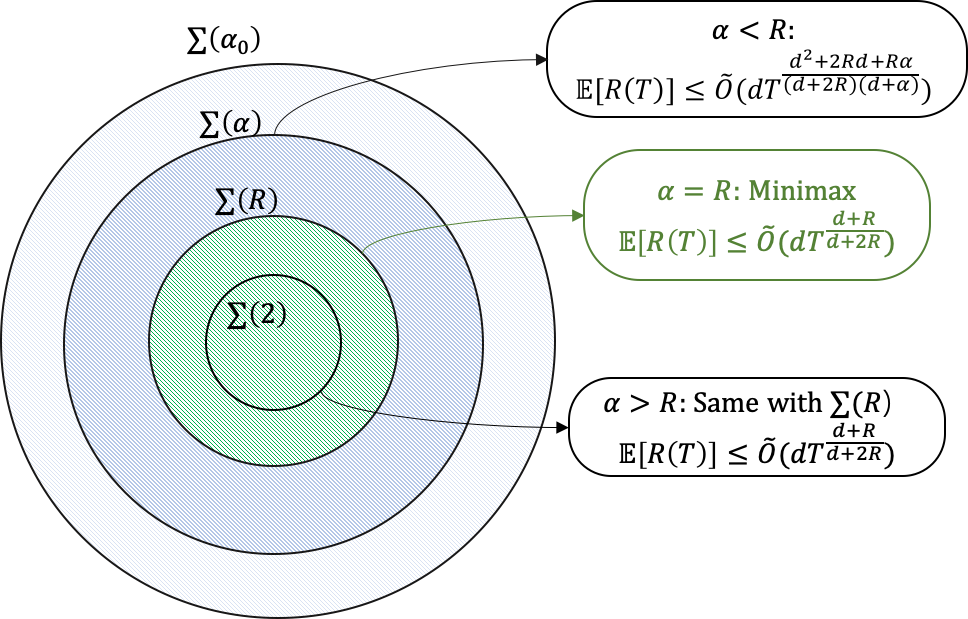}
      \caption{Illustration of adaptation to smoothness for continuous scale of H\"older spaces.}
      \label{fig: adaptation}
    \end{minipage}\hfill
    \begin{minipage}{0.48\textwidth}
      \centering
      \includegraphics[width=\textwidth]{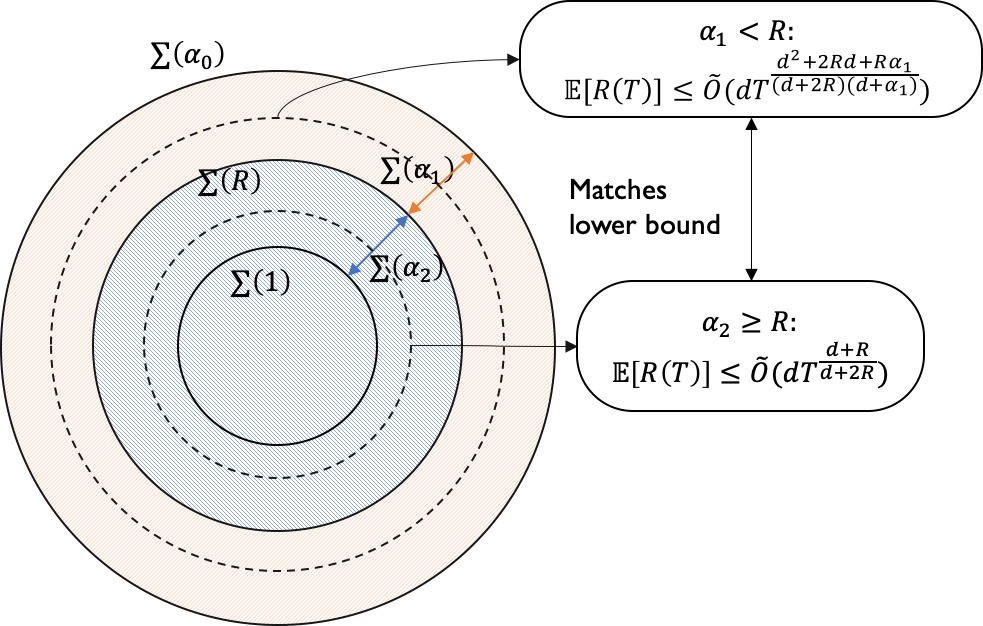}
      \caption{Illustration of values of exponents $\alpha_1, \alpha_2$ on which our proposed strategy matches the lower bound in \cite{locatelli2018adaptivity}. } 
    \label{fig: adaptation lower bound}
    \end{minipage}
 \end{figure}

\subsection{Comparison with Existing Lower Bound for Adaptation}
In this subsection, we compare the results in Theorem~\ref{thm: adaptation, smooth corral} to the existing lower bound in~\cite{locatelli2018adaptivity}. 
Theorem 3 of~\cite{locatelli2018adaptivity} state that given two smoothness values $\alpha_1<\alpha_2\leq 1$, if a strategy has expected regret $\EE[R_{\alpha_2}(T)]$ under exponent $\alpha_2$ that is $\tilde{\bigo}(T^{\frac{\xdim+\alpha}{\xdim+2\alpha}})$, then the regret of this strategy under the superset characterized by $\alpha_1$ is lower bounded by $\sup_{f\in \sum(\alpha_1, L)} \EE[R(T)] \geq \tilde\Omega(T R_{\alpha_2}(T)^{\frac{-\alpha_1}{\alpha_1+\xdim}})$, even if the strategy has access to both $\alpha_1$ and $\alpha_2$.

We make the following remark: for any pair of exponent values $(\alpha_1, \alpha_2)$ where $\alpha_1<R$ and $R\leq \alpha_2\leq 1$, the strategy proposed in Theorem~\ref{thm: adaptation, smooth corral} matches the lower bound except for log factors. We verify this by plugging in $\EE[R_{\alpha_2}(T)] =\tilde\bigo(T^\frac{\xdim+R}{\xdim+2R})$, omitting dependence on $\xdim$, to yield the lower bound on $\sum(\alpha_1, L)$ which is $\tilde\bigo(T^\frac{\xdim^2 + 2R\xdim+R{\alpha_1}}{(\xdim+2R)(\xdim+\alpha_1)})$. This is matched by our upper bound in equation~(\ref{eq: adapt rate smaller than R}), apart from log factors and $d$. An illustration is shown in Figure~\ref{fig: adaptation lower bound}.
In other words, the proposed algorithm can perform under unknown smoothness exponent and match the lower bound (available only for exponent values within $(0,1]$) on a subset of H\"older spaces.

\section{Conclusion}
The core of this paper is extending the assumption on function space from Lipschitz to H\"older spaces with higher-order smoothness in bandit optimization of black-box functions. We also study adaptation to the smoothness under this scope. 
The class of two-layer algorithms that we proposed consists of a Meta-algorithm with the choice of UCB (\cite{auer2002nonstochastic}) or Corral (\cite{agarwal2016corralling, pacchiano2020model}) and a set of misspecified bandit base-algorithms as arms. 
We derive regret upper bounds for $\alpha$-H\"older smooth functions with $\alpha>1$ that matches existing lower bounds in their dependence on $T$, the number of active queries, with straightforward generalization to larger $\alpha$.
Our framework provides useful insights in exploiting higher-order smoothness of reward functions for cumulative regret minimization, 
{because our two-layer structure allows base-algorithms to perform local exploration-exploitation tradeoff as opposed to the local pure exploration done for bandit optimization of $\alpha$-H\"older continuous functions.}
For adaptation to the smoothness exponent, we further previous works by deriving regret upper bound for adaptation to a continuous scale of H\"older spaces with exponent $\alpha$ {in a given range}. We show that by using bandit model selection algorithms, it can achieve the existing lower bound between two H\"older spaces,
even if the algorithm does not know both exponent values. 

Our work inspires several directions for the future. 
An intriguing direction is to study whether there exist gap-dependent bounds for the UCB-Meta algorithm, whose arms have non i.i.d rewards because they are bandit algorithms themselves. Such bounds could enable better rates for benign problem instances, for example with the growth conditions (mentioned in section 2). Another direction is the relaxation of the H\"older smooth assumption, to hold only around the maxima instead of everywhere on $\Xcal$, which is considered by prior works such as~\cite{auer2007improved, kleinberg2008multi, bubeck2010x}. Finally, it remains an open problem to establish the lower bound for adaptation when the smoothness exponents are larger than 1.  
\appendix
\onecolumn
\begin{center}
\resizebox{\linewidth}{!}{\mbox{\large\bfseries Appendix of \textit{Smooth Bandit Optimization: Generalization to H\"older Space}\par}}
\end{center}
\section{Auxiliary proofs for the main document}\label{app: auxilary proofs}
\subsection{Proof of Lemma~\ref{lemma: local bias}}\label{proof of lemma local bias}
\begin{proof}
    Recall the definition of H\"older smoothness: $\lvert f(x) - T_y^{l}(x)\rvert\leq L\lVert x-y\rVert_\infty^\alpha$. 
    For a hypercube $B$, $\lVert x-y \rVert_\infty \leq \Delta^{\frac{1}{\xdim}}, \forall x,y\in B$. By definition, when the function smoothness exponent $\alpha\in (1,2]$, $l=1$. 
    Notice that the Taylor polynomial of degree $l=1$ around $y$ is a linear   \footnote{We slightly abuse the notation and define short-hand notation $\la \theta,x\ra:= \theta_0+\sum_{i=1}^{\lineardim} \theta_i x_i$.
    } function of $x$: $T^{(l=1)}_y(x) = f(y) + \frac{\partial f}{\partial x_1}(y)(x_1-y_1) + \frac{\partial f}{\partial x_2}(y)(x_2-y_2) + \dots \frac{\partial f}{\partial x_d}(y)(x_d-y_d) = \la \theta, x\ra$. When $\alpha>2$, the Taylor polynomial can still be written as a linear function but of higher-dimensional feature map of $x$: $\phi: [0,1]^\xdim \rightarrow [0,1]^{\lineardim}$ which contains exponentiations of elements in $x$, using the operations defined for definition~\ref{def: Holder space}, $\phi(x) = \{x^s, \forall s, s.t. \abs{s}\leq l \}$. So:
    \begin{equation}
    \lineardim = \abs{\{s: 1\leq \abs{s}\leq l \}} = \sum_{1\leq j\leq l} {{j+d-1}\choose{d-1}} = \bigo({\xdim}^l)
    \end{equation}
    When $l=1$, it is equivalent to defining $\phi(x) = x$. 
    The parameter $\theta$ is determined by the derivatives of $f$ at $y$ and the value of $y$.
    Therefore, we know locally there exists a linear parameter in dimension $\theta^* = \argmin_\theta \lVert f - \phi(x)^T\theta\rVert_\infty, x\in B$, such that $\lVert f - \la\theta^*, \phi(x)\ra \rVert_\infty \leq \epsilon = {L} \Delta^{\frac{\alpha}{d}}, \forall x\in B$. Also, note that $\norm{\phi(x)}^2\leq {\lineardim}^2$ according to definition. 
    When the exponent $\alpha\in(0,1]$, $l$ is 0 and the Taylor polynomial is simply a constant. Therefore the same argument holds for $\theta^*$ for example when $\theta^*_{1},\dots,\theta^*_{d} =0$ (a constant function).  
    \end{proof}
\subsection{Proof of Theorem~\ref{thm: main regret}}\label{sec: proof of thm: main regret}
\begin{proof}
    Throughout this proof, we assume that the assumptions A1$\sim$3 hold. This proof is modified from that in~\cite{dani2008stochastic}. Some techniques are from~\cite{abbasi2011improved}. We only present the parts which we change. 
    First we proof the following bound on simple regret at each step:
    \begin{equation}\label{eq: simple pseudo-regret}
        r_t \leq 2\sqrt{\beta_t} \lVert A_t^{-1/2} x \rVert + 2\epsilon \sum_{\tau = 1}^{t-1} \lVert x^T A_t^{-1}x_\tau \rVert.
    \end{equation}
    And then we will bound the sum of these two terms separately. In order to proof inequality~\ref{eq: simple pseudo-regret}, we start from an important auxiliary theorem of confidence bound on $\theta^*$, Theorem~\ref{thm: main confidence}. 
    \begin{theorem}\label{thm: main confidence}
        Let $\beta_t = C\sigma^2d\ln(1+ t\kappa^2/d)\ln(\frac{2t^2}{\delta})\left (=\bigo(d\ln(t)\ln(\frac{t^2}{\delta}))\right)$ for a sufficiently large constant C, then with probability $1-\delta$, $\theta^*$ is contained in the confidence set:
        \[\tilde{C}_t = \{\hat\theta_t + \sqrt{\beta_t}A_t^{-1/2}z_d - A_t^{-1}(\sum_{s=1}^{t-1} b_s x_s), \lVert z_d \rVert_2\leq 1\},\]
        and as a result, 
        \[\la x, \theta^*\ra\leq\la x, \hat\theta_t \ra + \sqrt{\beta_t} \lVert A_t^{-1/2} x \rVert + \epsilon \sum_{s=1}^{t-1} \lvert x^TA_t^{-1}x_s \rvert.\]
    \end{theorem}
    The proof of Theorem~\ref{thm: main confidence} is in Appendix~\ref{sec: proof of thm: main confidence}.
    Now, if $\theta^*\in \tilde{C}_t$, we have 
    \begin{equation*}
        \begin{split}
        r_t &= \la x^*, \theta^*\ra - \la x_t, \theta^* \ra \\
        &\leq \la x^*, \theta^*\ra - UCB_t(x^*) + UCB_t(x_t) - \la x_t, \theta^* \ra \\
        &\leq UCB_t(x_t) - \la x_t, \theta^* \ra \\
        &\leq 2\sqrt{\beta_t} \lVert A_t^{-1/2} x_t \rVert + 2\epsilon \sum_{s=1}^{t-1} \lvert x_t^TA_t^{-1}x_s \rvert.
        \end{split}
    \end{equation*}
    The first inequality is because our algorithm will only choose $x_t$ when $UCB_t(x_t) \geq UCB_t(x^*)$. The last inequality holds because
    \begin{equation*}
        \begin{split}
        \la x, \theta^*\ra &\geq \la x, \hat\theta_t \ra + \min_{z_d \in B_2^d} \sqrt{\beta_t} \la x, A_t^{-1/2}z_d \ra - \sum_{s=1}^{t-1}  b_s x^TA_t^{-1}x_s \nonumber \\
        &\geq \la x, \hat\theta_t \ra - \sqrt{\beta_t} \lVert A_t^{-1/2} x \rVert - \sum_{s=1}^{t-1}  b_s x^TA_t^{-1}x_s \nonumber  \\
        &\geq UCB_t(x) - 2\sqrt{\beta_t} \lVert A_t^{-1/2} x \rVert - 2\epsilon \sum_{s=1}^{t-1} \lvert x^TA_t^{-1}x_s \rvert.
        \end{split}
    \end{equation*}
    By assumption on the mean reward function value, the absolute value of instant pseudo-regret $\lvert r_t \rvert$ is bounded by $1+\epsilon$. Therefore, combining inequality~(\ref{eq: simple pseudo-regret}) and $r_t\leq 2 + 2\epsilon$, we have that\footnote{$a \land b =\min(a,b)$}
    \begin{equation}\label{eq: decompose of simple regret}
        \begin{split}
        r_t &\leq (2+2\epsilon)\land \left(2\sqrt{\beta_t} \lVert A_t^{-1/2} x_t \rVert + 2\epsilon \sum_{\tau = 1}^{t-1} \lVert x_t^T A_t^{-1}x_\tau \rVert\right) \\
        &\leq 2\underbrace{\left(1 \land \sqrt{\beta_t} \lVert A_t^{-1/2} x_t \rVert \right)}_\text{\#1} + 2\underbrace{\epsilon \sum_{\tau = 1}^{t-1} \lVert x_t^T A_t^{-1}x_\tau \rVert}_\text{\#2} + 2\epsilon.
        \end{split}
    \end{equation}
    Sum of term $\#1$ is bounded using bound~(\ref{eq: original bound on w_t}) and  Cauchy Schwartz inequality:
    \begin{equation}
    2\sum_{t=1}^T (1\land \sqrt{\beta_t} \lVert A_t^{-1/2} x_t \rVert) \leq 2\sqrt{T\beta_T\sum_{t=1}^T (1\land\lVert x_t^T A_t^{-1} x_t \rVert})  = \sqrt{8d\beta_T T\ln(1 + T\kappa^2/d)}.
    \end{equation}
    For sum of term $\#2$, we first have
    \begin{equation*}
        \begin{split}
        \sum_{\tau=1}^{t-1}x_t^T A_t^{-1}x_\tau 
        &\leq \sqrt{t \sum_{\tau=1}^{t-1} x_t^T A_t^{-1}x_\tau x_\tau^T A_t^{-1}x_t} \\
        &= \sqrt{t x_t^T A_t^{-1}(\sum_{\tau=1}^{t-1} x_\tau x_\tau^T) A_t^{-1}x_t}\\
        &\leq \sqrt{t x_t^T A_t^{-1}(\sum_{\tau=1}^{t-1} x_\tau x_\tau^T) A_t^{-1}x_t + x_t^T A_t^{-1}A_t^{-1}x_t} \\
        &=\sqrt{t x_t^T A_t^{-1}(\sum_{\tau=1}^{t-1} x_\tau x_\tau^T + I_d) A_t^{-1}x_t}=\sqrt{t x_t^T A_t^{-1}x_t}.
        \end{split}
    \end{equation*}
    Then the sum $\sum_{t=1}^T(\sum_{\tau=1}^{t-1}x_t^T A_t^{-1}x_\tau)$ can be bounded by:
    \begin{equation}\label{eq: decompose of bias-regret}
        \begin{split}
        \sum_{t=1}^T(\sum_{\tau=1}^{t-1}x_t^T A_t^{-1}x_\tau)&\leq \sum_{t=1}^T(\sqrt{t x_t^T A_t^{-1}x_t}) \nonumber\\
        &\leq \sqrt{(\sum_{t=1}^Tt)(\sum_{t=1}^T x_t^T A_t^{-1}x_t)}.
        \end{split}
    \end{equation}
    Now, we need to bound $\sum_{t=1}^T x_t^T A_t^{-1}x_t$ with inequality~(\ref{eq: original bound on w_t}). We know that $A_t^{-1}$ is a full-rank matrix. Therefore, denote its eigenvalues and eigenvectors as $\lambda_1\dots \lambda_d, v_1\dots v_d$. Then\footnote{This proof is extracted from a remark in proof of Theorem 3 in~\cite{abbasi2011improved}}
    \begin{equation*}
        \begin{split}
        x_t^T A_t^{-1}x_t &= \left(c_1 v_1+\dots +c_d v_d\right)^TA_t^{-1}\left(c_1 v_1+\dots + c_d v_d\right) \\
        &=c_1^2\lambda_1 +\dots + c_d^2 \lambda_d \\
        &\leq \lambda_{\max}(A_t^{-1}) \norm{x_t}^2 = \frac{\kappa^2}{\lambda_{\min}(A_t)} \\
        &\leq \frac{\kappa^2}{\lambda_{\min}(I_d)+ \lambda_{\min}(X_t^T X_t)} \leq \kappa^2.
        \end{split}
    \end{equation*}
    The second last inequality holds due to Weyl's inequality. Therefore, 
    \begin{equation}\label{eq: bound on sum wt}
        \begin{split}
        \sum_{t=1}^T x_t^T A_t^{-1} x_t &\leq \kappa^2 \sum_{t=1}^T (x_t^T A_t^{-1} x_t \land 1) \nonumber \\
        &\leq \kappa^2 (2d\ln(1 + T\kappa^2/d)).
        \end{split}
    \end{equation}
    Putting the above together, 
    \begin{equation}\label{eq: bound on linear UCB term two}
        \begin{split}
        \sum_{t=1}^T \left(2\epsilon \sum_{\tau=1}^{t-1} x_t^T A_t^{-1}x_\tau\right) 
        &\leq 2\epsilon \sqrt{(\sum_{t=1}^T t) (\sum_{t=1}^T x_t^T A_t^{-1}x_t)}\\
        &\leq 2\epsilon T \kappa \sqrt{2d\ln(1 + T\kappa^2/d)}. 
        \end{split}
    \end{equation}
    Finally, plugging in $\kappa^2 = d$ gives the final results.
 \end{proof}

\subsubsection{Proof of Theorem~\ref{thm: main confidence}}\label{sec: proof of thm: main confidence}
\begin{proof}
    Let $\hat\theta_t = A_t^{-1}X_t^T y$ denote the regularized least square estimator at time $t$. Matrix $X_t$ has dimension $(t-1)\times d$, where each row is a past action (until time $t$). We first define an unobserved variable $\tilde\theta_t$:
    \begin{equation}\label{def: bias}
    \tilde{\theta_t} = A_t^{-1} X_t^T (X_t \theta^* + \eta_t) = \hat\theta_t - A_t^{-1} X_t^T b_t,
    \end{equation}
    here we abuse the notations and let $\eta_t$ and $b_t$ be the $(t-1) \times 1$ vector containing noise and bias of each time. 
    Then we define the following confidence ellipsoid centered at $\tilde\theta_t$:
    \begin{equation}\label{def: new ellipsoid}
    C_t = \{\theta: (\theta - \tilde\theta_t)^TA_t(\theta - \tilde\theta_t) \leq \beta_t\},
    \end{equation}
    and prove the following lemma as an analog to Theorem 5 of~\cite{dani2008stochastic}:
    \begin{lemma}\label{lemma: high prob bound old confidence ball}
        The true linear parameter $\theta^*$ is contained in ellipsoid $C_t$, specifically, $\PP(\forall t, \theta^* \in C_t) \geq 1 - \delta$.
    \end{lemma}
    The proof is in Appendix section~\ref{sec: proof of high prob bound old confidence ball}. However, we do not observe the vector $b_t$, so we cannot calculate $C_t$ in our algorithm. So instead, we define a larger $\tilde{C}_t$ that contains $C_t$, which will naturally contains $\theta^*$ with high probability. To construct $\tilde C_t$, we first re-write $C_t$ as
    \begin{equation}\label{def: old confidence set}
        C_t = \{\tilde\theta_t + \sqrt{\beta_t}A_t^{-1/2}z_d, \lVert z_d \rVert_2\leq 1\},
    \end{equation}
    then plug in equation~(\ref{def: bias}) to yield:
    \begin{equation}
        \begin{split}
        \tilde\theta_t + \sqrt{\beta_t}A_t^{-1/2}z 
        & = \hat\theta_t + \sqrt{\beta_t}A_t^{-1/2}z - A_t^{-1}X_t^T b_t\\
        & = \hat\theta_t + \sqrt{\beta_t}A_t^{-1/2}z - A_t^{-1}(\sum_{s=1}^{t-1} b_s x_s).
        \end{split}
    \end{equation}
    
    Therefore, we know that with high probability, 
    \begin{equation}\label{def: main confidence set}
        \theta^* \in \tilde{C}_t = \{\hat\theta_t + \sqrt{\beta_t}A_t^{-1/2}z_d - A_t^{-1}(\sum_{s=1}^{t-1} b_s x_s)\}.
    \end{equation}
    Therefore, we have a computable confidence bound for $x$: 
    \begin{equation}\label{def: main ucb}
        \begin{split}
        UCB_t(x) &= \max_{\theta\in\tilde{C}_t} \la x, \theta \ra \\
        &= \la x, \hat\theta_t \ra + \max_{z_d \in B_2^d} \sqrt{\beta_t} \la x, A_t^{-1/2}z_d \ra - \sum_{s=1}^{t-1}  b_s x^TA_t^{-1}x_s \\
        &\leq\la x, \hat\theta_t \ra + \sqrt{\beta_t} \lVert A_t^{-1/2} x \rVert - \sum_{s=1}^{t-1}  b_s x^TA_t^{-1}x_s\\
        &\leq \la x, \hat\theta_t \ra + \sqrt{\beta_t} \lVert A_t^{-1/2} x \rVert + \epsilon \sum_{s=1}^{t-1} \lvert x^TA_t^{-1}x_s \rvert.
        \end{split}
    \end{equation}
    The first inequality is derived by Cauchy Schwartz inequality and the fact that $z_d$ is in unit ball.
    \end{proof}

\subsubsection{Proof of Lemma~\ref{lemma: high prob bound old confidence ball}}\label{sec: proof of high prob bound old confidence ball}
\begin{proof}
    Lemma~\ref{lemma: high prob bound old confidence ball} is a parallel to Theorem 5 in~\cite{dani2008stochastic}, with the difference of sub-gaussian noise, ellipsoid centre $\tilde{\theta}_t$ and misspecification in observation. The key idea is the same, namely to use induction to bound the growth of $Z_t=(\theta^* - \tilde\theta_t)^T A_t(\theta^* - \tilde\theta_t)$ and proof that $Z_t\leq \beta_t$, i.e. the $\theta^*$ is contained in $C_t$, at each time step $t$. 
    The following analysis used the same notations and definitions as section $5.2$ in~\cite{dani2008stochastic} unless otherwise specified. Under Lemma~\ref{lemma: high prob bound old confidence ball}'s definition of confidence set $C_t$, we have that: 
    \begin{align}
        H_t &= A_t(\tilde\theta_t - \theta^*) = X_t^T\eta_t - \theta^*, \label{eq: H_t} \\
        Z_t &= (\theta^* - \tilde\theta_t)^T A_t(\theta^* - \tilde\theta_t)
        = H_t^T A_t^{-1} H_t.
    \end{align}
    Equation~\ref{eq: H_t} holds because of this key property:
    \begin{equation}
    \tilde\theta_t: A_t \tilde\theta_t = X_t^T X_t\theta^* + X_t^T\eta_t.
    \end{equation}
    And the rest of the proof in~\cite{dani2008stochastic} should go through by substituting $Y_t$ with $H_t$ (defined above) and $\hat\mu$ with our definition of $\tilde{\theta}$ (centre of the confidence ellipsoid). 
    Except, to accommodate the sub-gaussian noise assumption that replaces their bounded noise assumption, we have to make two changes in the proof. Both are in analyzing the growth of $Z_t$ in the induction. 
    Recall that~\cite{dani2008stochastic} proved this relation: 
    \begin{equation}
        Z_t\leq Z_1 + 2\sum_{\tau=1}^{t-1} \eta_t \frac{x_t^T(\tilde\theta_t - \theta^*)}{1+w_t^2} + \sum_{\tau=1}^{t-1} \eta_\tau^2\frac{w_{\tau}^2}{1+w_{\tau}^2}.
    \end{equation}
    We first look at the concentration of the sum of martingale difference sequence that makes up $Z_t$: same with~\cite{dani2008stochastic}, define $M_t = 2\eta_t \frac{x_t^T(\tilde\theta_t - \theta^*)}{1+w_t^2}$ where $w_t \stackrel{\triangle}{=}\sqrt{x_t^TA_t^{-1}x_t}$. According to our assumption, the noise sequence is a sub-gaussian martingale difference sequence with parameter $\sigma^2$. 
    Therefore, $M_t$ is a sub-gaussian martingale difference sequence. Specifically, we know that the square of subgaussian parameter is $ 4\sigma^2(\frac{\lvert x_t^T(\tilde\theta_t - \theta^*)\rvert}{1+w^2_t})^2$. By definitions we know that $M_t\given \Hcal_t$ is $(\nu^2_t=4\sigma^2(\frac{\lvert x_t^T(\tilde\theta_t - \theta^*)\rvert}{1+w^2_t})^2,a_t = 0)$ sub-exponential(definition 2.7 in~\cite{wainwright2019high}) and therefore the sum $\sum_{\tau=1}^tM_\tau$ is also sub-exponential, with parameters $(\sqrt{\sum_{\tau=1}^t \nu_\tau^2}, a = \max_{\tau}a_\tau=0)$(Theorem 2.19 (1) in~\cite{wainwright2019high}). The following inequality is conditioned on the fact that from time $\tau=1\dots t$, $\theta^*$ is contained in $C_{\tau}$ (by the induction). 
    \begin{equation*}
        \begin{split}
            \sum_\tau^t \nu_t^2 &= 4\sigma^2 \sum_{\tau=1}^t (\frac{\lvert x_\tau^T(\tilde\theta_\tau - \theta^*)\rvert}{1+w^2_\tau})^2 \\
            &\leq 4\sigma^2 \sum_{\tau=1}^t (\frac{\sqrt{\beta_\tau}w_\tau}{1+w^2_\tau})^2 \\
            &\leq 4\sigma^2 \sum_{\tau=1}^t \beta_\tau (\min(1/2, w_\tau))^2 \\
            &\leq 4\sigma^2 \sum_{\tau=1}^t \beta_\tau \min(1/4, w_\tau^2) \\
            &\leq 4\sigma^2 \beta_t \sum_{\tau=1}^t \min(1, w_\tau^2) \\ 
            &\leq 4\sigma^2 \beta_t \left(2d\ln(1 + t\kappa^2/d)\right)\hfill\text{See bound~\ref{eq: original bound on w_t}} \\ 
            &= 8\sigma^2 d\beta_t\ln(1 + t\kappa^2/d).
        \end{split}
    \end{equation*}
    The proof for the first three inequalities is the same as Lemma 7 and section 5.2.1 in~\cite{dani2008stochastic}. 
    Then we apply a Bernstein-type concentration bound for sub-exponential martingale difference sequence (Theorem 2.19 (2) in~\cite{wainwright2019high}). Plugging in the values of $a$ and ${\sum_{\tau=1}^t \nu_\tau^2}$, we have that
    \begin{equation}\label{eq: bound on M_tau}
        \begin{split}
        \PP(\vert \sum_{\tau=1}^{t-1} M_\tau \vert\geq s) &\leq 2\exp(\frac{-s^2}{2\sum_{\tau=1}^{t-1} \nu_t^2})\\
        &\leq 2\exp(\frac{-s^2}{16\sigma^2 d\beta_t\ln(1 + (t-1)\kappa^2/d)}) \\
        &\stackrel{s = \frac{\beta_t}{2}}{=} 2\exp\left(\frac{-\beta_t}{64\sigma^2 d \ln(1+(t-1)\kappa^2/d)} \right)\\
        &\leq \frac{\delta}{2t^2} \hfill \text{ (Needed for union bound over all times)}.
        \end{split}
    \end{equation}
    Therefore, as long as $\beta_t$ is larger or equal to $64\sigma^2d\ln(1+ (t-1)\kappa^2/d)\ln(\frac{4t^2}{\delta})$, $\sum_{\tau=1}^{t-1} M_\tau \leq \frac{\beta_t}{2}$ with probability larger or equal to $1-\frac{\delta}{2t^2}$. 
    
    The second change is for the third quantity that makes up 
    $Z_t$: $\sum_{\tau=1}^{t-1} \eta_\tau^2\frac{w_{\tau}^2}{1+w_{\tau}^2}$. We need to bound $\max_{\tau \leq t-1}\eta_\tau^2$ with high probability.  
    By algebra calculations, we know that $\eta_\tau^2$ is sub-exponential with parameters {$(\nu=32\sigma^4, a=4\sigma^2)$}\footnote{For this part, we borrowed the proof from Example 2.8 in~\cite{wainwright2019high} and \url{http://proceedings.mlr.press/v33/honorio14-supp.pdf}}. We can apply union bound with the tail bound of sub-exponential variables:
    \begin{equation*}
        \begin{split}
            \PP(\max_{\tau \leq t-1} (\eta_\tau^2 - \EE[\eta^2])\geq z) &\leq \sum_{\tau=1}^{t-1}\PP((\eta^2_\tau - \EE[\eta^2]) \geq z) \\
            &\leq (t-1)\exp(-\frac{z}{2a}) \text{ (Proposition 2.9 in~\cite{wainwright2019high})} \\
            &\leq \frac{\delta}{2t^2} \text{ (Needed for union bound over all times)}.
        \end{split}
    \end{equation*}
    Set $z = 8\sigma^2\ln(\frac{2t^3}{\delta})$ so that $\PP(\max_{\tau \leq t-1} \eta_\tau^2 - \EE[\eta^2] \leq z) = \PP(\max_{\tau \leq t-1} \eta_\tau^2  \leq z + \EE[\eta^2])\geq 1- \frac{\delta}{2t^2}$. By the fact that $\EE[\eta]=0$, $\EE[\eta^2] = \text{Var}(\eta)\leq \sigma^2$, which is a property of subgaussian variables. So $\PP(\max_{\tau \leq t-1} \eta_\tau^2  \leq z + \sigma^2)\geq 1- \frac{\delta}{2t^2}$.
    The following holds with probability larger than $1-\frac{\delta}{2t^2}$:
    \begin{equation*}
    \begin{split}
        \sum_{\tau=1}^{t-1} \eta_\tau^2\frac{w_{\tau}^2}{1+w_{\tau}^2} &\leq (\max_{\tau \leq t-1}\eta_\tau^2) \sum_{\tau=1}^{t-1} \min(w_\tau^2, 1) \\
        &\leq (\max_{\tau \leq t-1}\eta_\tau^2)2d\ln(1 + t\kappa^2/d) \\
        & = (8\sigma^2\ln(\frac{2t^3}{\delta})+\sigma^2) 2d\ln(1 + (t-1)\kappa^2/d) \\
        & = 8\sigma^2 (\ln(\frac{2t^3}{\delta})+\frac{1}{8}) 2d\ln(1 + (t-1))\kappa^2/d) \\
        & = 16\sigma^2d\ln(1 + (t-1)\kappa^2/d)\left(\ln(\frac{2t^3}{\delta})+\frac{1}{8}\right).
    \end{split}
    \end{equation*}
    
    Except the two changes above, one last thing to note is the quantity $Z_1$ analyzed at the end of proof of Lemma 12 in~\cite{dani2008stochastic}. In our assumption of the reward function value, we conclude that
    \begin{equation*}
        \begin{split}
            Z_1 &= (\theta^* - 0)^T I (\theta^* - 0)= \norm{\theta^*}^2 \\
            &= \sum_{i=1}^d (e_i^T \theta^*)^2 \hfill \text{ ($e_i$ is base vector of dimension $i$, note that $e_i\in \Xcal$)} \\
            &\leq d(1+\epsilon)^2.
        \end{split}
    \end{equation*}
    
    As a result, if it is satisfied that $Z_t \leq Z_1 + \beta_t/2 + 16\sigma^2d\ln(1 + (t-1)\kappa^2/d)(\ln(\frac{2t^3}{\delta})+\frac{1}{8})  \leq \beta_t$,
    which enables the induction in Lemma 14 in~\cite{dani2008stochastic}, then the rest of the proof should go through smoothly. We argue that setting $\beta_t = C \sigma^2 d \ln(t)\ln(\frac{4t^2}{\delta})$ for a large enough constant $C$ suffices. This is under the reasonable assumption that $\epsilon$ is $\bigo(1)$ and $\sigma$ is a constant\footnote{Recall that according to Lemma~\ref{lemma: local bias}, $\epsilon$ is bounded by the Lipschitz constant $L$ and is therefore $\bigo(1)$}.
    
    It is worth mentioning\footnote{This remark is made by~\cite{abbasi2011improved}.} that~\cite{dani2008stochastic} requires the relationship between $t$ and $\delta$ to be approximately $0<1.05\delta\leq t^2$, hence their requirement\footnote{However, we believe that this should not translate to a constraint on $t$, but on $\delta$ instead. Because $Z_t \leq \beta_t$ is required for every step $t$ to complete the induction, so if it only holds for large $t$ then the induction will fail as well. } of ``for sufficiently large T" in Theorem 1 and 2. This is because of the last step of their induction proof for Theorem 5 requires: $Z_t \leq d + \beta_2/2 + 2d\ln(t)\leq \beta_t$. 
    In our setting, the requirement in induction
    translates to this (second) constraint(plugging in $\kappa^2 = d$): $\beta_t\geq 2d(1+\epsilon)^2 + 32\sigma^2 d \ln(t)(\ln(\frac{2t^3}{\delta}+\frac{1}{8}))$. 
    Recall the first constraint on $\beta_t$ is
    $\beta_t\geq 64\sigma^2d\ln(t)\ln(\frac{4t^2}{\delta})$, from bound~(\ref{eq: bound on M_tau}). 
    Therefore, $C$ should first satisfy $C\geq 64$ and for the second constraint we need\footnote{This is from the second constraint: $C \sigma^2 d \ln(t)\ln(\frac{4t^2}{\delta}) \geq \frac{2}{3}C \sigma^2 d \ln(t)\ln(\frac{2t^3}{\delta}) \geq 2d(1+\epsilon)^2 + 32\sigma^2 d \ln(t)(\ln(\frac{2t^3}{\delta}+\frac{1}{8}))$.}: 
    $C\geq \frac{3(1+\epsilon)^2}{4(\ln(2))^2\sigma^2} + \frac{3}{2\ln(2)}+ 48$. Therefore, the lower bound of $C$ should depend on values of $\epsilon$ and $\sigma^2$. 
    The choice of $C=128$ in the main theorem is an example that requires approximately $\frac{1+\epsilon}{\sigma}\leq 7$. 
\end{proof}

\subsection{Proof of Theorem~\ref{thm: meta main regret}}\label{sec: proof of thm: meta main regret}
{Let us treat the number of bins/local algorithms $n$ as the input parameter to the algorithm. The regret bound of UCB-Meta (equation~\ref{eq: meta regret}) should be independent of the input dimension $\xdim$, given the dimension of the linear model $\lineardim$. Therefore, throughout this proof we will abuse the notations and let $d$ denote the linear model dimension for simplicity. }
\begin{proof}
First, we define the ``good event" $E_{good}$ as an event where all confidence bound holds for all bins at all times. For a fixed bin, if $\PP(\theta^* \notin \tilde{C}_t, \exists t) \leq \delta/n$, as set in the algorithm, where $\tilde{C}_t = \{\hat\theta_t + \sqrt{\beta_t}A_t^{-1/2}z_d - A_t^{-1}(\sum_{s=1}^{t-1} b_s x_s)\}$ (Theorem~\ref{thm: main confidence}),
then by union bound, $\PP(\theta^*_k \notin \tilde{C}_{k,t}, \exists k ) \leq  \delta$, where $\tilde {C}_{k,t}$ is the confidence ellipsoid of bin $k$ at time $t$. The good event is 
$E_{good} = \{\forall t, \forall k\in[n], \theta^*_k \in \tilde{C}_{k,t}\}$. 
It happens with probability $\PP(E_{good})\geq 1-\delta$, and the following proof will condition on it. 

Here are some useful notations that make the proof easier to read: let ${N^k(t)}$ denote the number of times base-algorithm $\Acal^{local}_k$ has been selected by(including) time $t$; let $k(t)$ denote the bin selected at time $t$; let $x_t$ denote the action selected at time $t$; let $\{\beta_{k, \cdot}\}$, $\{A_{k, \cdot}\}$ and $\{\hat\theta_{k,\cdot}\}$ denote the set of parameters kept by that base-algorithm~$\Acal^{local}_k$. 

The upper confidence bound on value of the local linear function achieved by sub-algorithms at round $t$ is defined as 
$UCB_{k(t),t}(x) = \la x, \hat\theta_{{k,N^k(t)}}\ra + \sqrt{\beta_{k,N^{k}(t)}}\lVert A_{k, N^{k}(t)}^{-1/2} x\rVert + \epsilon\sum_{\tau=1}^{{N^{k}(t)}-1}\vert x^T A_{N^{k}(t)}^{-1}x_\tau\vert$ for any action $x\in B_k$.
Using the proof of Theorem~\ref{thm: main regret}, the good event hence indicates that for the base-algorithm selected at time $t$ and any action $x\in B_{k(t)}$:
$$UCB_{k(t),t}(x) - 2\sqrt{\beta_{k,N^{k}(t)}}\lVert A_{k, N^{k}(t)}^{-1/2} x\rVert - 2\epsilon\sum_{\tau=1}^{{N^{k}(t)}-1}\vert x^T A_{N^{k}(t)}^{-1}x_\tau\vert \leq \la x, \theta^*_k\ra \leq  UCB_{k(t),t}(x).$$
By Lemma~\ref{lemma: local bias}, the expected local function value $f(x)$ is bounded by $$UCB_{k(t),t}(x) - 2\sqrt{\beta_{k,N^{k}(t)}}\lVert A_{k, N^{k}(t)}^{-1/2} x\rVert - 2\epsilon\sum_{\tau=1}^{{N^{k}(t)}-1}\vert x^T A_{N^{k}(t)}^{-1}x_\tau\vert -\epsilon\leq f(x)\leq UCB_{k(t),t}(x)+ \epsilon. $$
A common way to bound pseudo regret for stochastic bandit is via Wald’s equality: $R_T = \sum_{k=1}^n \Delta_k \EE[\tau_k(T)]$ where $\tau_k(T)$ is the number of times arm $k$ gets pulled until time $T$, and $\Delta_k$ is the reward gap. We cannot trivially follow this, because the rewards of each bins are no longer i.i.d. Instead, we use this gap-independent decomposition for each bin $k$:
\begin{equation}
    \begin{split}
    R_k &= \sum_{t: \text{bin}_t = k}(f^* - f_{x_t\in B_k}(x_t))\\
    &=\sum_{t: \text{bin}_t = k}\left(f^* - UCB_{\Acal_{k(t)},t}+ UCB_{\Acal_{k(t)},t} - f(x_t)\right) \\
    &=\sum_{t: \text{bin}_t = k}\left(f^* - UCB_{\Acal_{k(t)},t} + UCB_{k(t),t}(x_t)+ \epsilon - f(x_t)\right) \\
    &\leq \sum_{t: \text{bin}_t = k}\left(UCB_{k(t),t}(x_t)+ \epsilon - f(x_t)\right) \\
    & \leq \sum_{t: \text{bin}_t=k}\left(2\sqrt{\beta_{k,N^{k}(t)}}\lVert A_{k, N^{k}(t)}^{-1/2} x_t\rVert + 2\epsilon\sum_{\tau=1}^{{N^{k}(t)}-1}\vert x_t^T A_{N^{k}(t)}^{-1}x_\tau\vert +2\epsilon\right) \\
    &= \sum_{s=1}^{N^k(T)}\left(2\sqrt{\beta_{k,s}}\lVert A_{k, s}^{-1/2} x_t\rVert + 2\epsilon\sum_{\tau=1}^{s-1}\vert x_t^T A_{s}^{-1}x_\tau\vert +2\epsilon\right).
    \end{split}
\end{equation}
The first inequality holds because of the algorithm's bin selection rule: if bin $B_k$ is chosen then $f^*\leq UCB_{k*,t} \leq UCB_{k(t)}$. 
By the bounded function value assumption, $f^* - f_{x_t\in B_k}(x_t) \leq 2$, therefore:
\begin{equation}\label{eq: meta regret decomposition - one arm}
    \begin{split}
    R_k &\leq \sum_{s=1}^{N^k(T)}\left(2\sqrt{\beta_{k,s}}\lVert A_{k, s}^{-1/2} x_t\rVert + 2\epsilon\sum_{\tau=1}^{s-1}\vert x_t^T A_{s}^{-1}x_\tau\vert +2\epsilon\right) \land 2\\
    &\leq \sum_{s=1}^{N^k(T)}\left(\underbrace{2\left(\sqrt{\beta_{k,s}}\lVert A_{k, s}^{-1/2} x_t\rVert \land 1\right)}_{\#1} + \underbrace{2\epsilon\sum_{\tau=1}^{s-1}\vert x_t^T A_{s}^{-1}x_\tau\vert}_{\#2} \right) + 2\epsilon N^k(T).
    \end{split}
\end{equation}
\subsubsection{High probability regret bound part I (term $\#1$)}\label{sec: meta regret part I}
First we establish this bound the same way as~\cite{dani2008stochastic}. Namely, for any local misspecified linear bandit algorithm that is ran $T$ times with data $(x_t, y_t)_{t=1\dots T}$,
\begin{equation}\label{eq: original bound on w_t}
    \begin{split}
    \sum_{t=1}^T \lVert x_t^T A_t^{-1} x_t \rVert \land 1
    &\leq 2\ln(\prod_{t=1}^T(1+x_t^T A_t^{-1} x_t)) \\
    &= 2\ln(\prod_{t=1}^T\frac{\det(A_{t+1})}{\det(A_t)}) \\
    &= 2\ln(\frac{\det A_{T+1}}{\det A_1}) \leq 2\ln((1 + T\kappa^2/d)^{d}) \\
    &= 2d\ln(1 + T\kappa^2/d),
    \end{split}
\end{equation}
where we used Lemma~\ref{lemma: determinant-trace}.
Now we can bound term \#1 using bound~(\ref{eq: original bound on w_t}). 
\begin{equation*}
\begin{split}
&\sum_{s=1}^{N^k(T)} 2(\sqrt{\beta_{k,s}}\lVert A_{k, s}^{-1/2} x_t\rVert\land 1) \\
&\leq \sqrt{N^k(T)\sum_{s=1}^{N^k(T)} 4(\beta_{k,s} \lVert x_{k,s}^TA_{k,s}^{-1}x_{k,s}\rVert \land 1)} \\
&\leq \sqrt{4\beta_{k,N^k(T)} N^k(T)\sum_{s=1}^{N^k(T)} \lVert x_{k,s}^TA_{k,s}^{-1}x_{k,s}\rVert \land 1} \\
&=\sqrt{4\beta_{k,N^k(T)} N^k(T)2\ln\left(\prod_{s=1}^{N^k(T)} (1+ x_{k,s}^TA_{k,s}^{-1}x_{k,s}) \right)}\\ 
&=\sqrt{4\beta_{k,N^k(T)} N^k(T)2\ln\left(\frac{\det(A_{N^k(T)+1})}{\det(A_1)}\right)} \\
&= \sqrt{8d\beta_{k,N^k(T)} N^k(T)\ln\left(1 + N^k(T)\kappa^2/d\right)}\\
&\stackrel{\kappa^2=d}{=} \sqrt{8d\beta_{k,N^k(T)} N^k(T)\ln\left(1 + N^k(T)\right)}.
\end{split}
\end{equation*}
\begin{lemma}\label{lemma: determinant-trace}
    For $t\geq 1$,  $1+ x_t^TA_{t}^{-1}x_t = \det(A_{t+1})/\det(A_t)$. Also, $\det(A_t)\leq (1 + (t-1)\kappa^2/d)^{d}$. 
\end{lemma}
\begin{proof}[Proof of Lemma~\ref{lemma: determinant-trace}]
\begin{equation*}
\begin{split}
\det(A_{t+1}) &= \det(A_t(I_d +A_t^{-1} x_t x^T_t)) =  \det(A_t)\det(I_d + A_t^{-1}x_t x^T_t) \\
&= \det(A_t) \det(I_1+ x_t^T A_t^{-1}x_t) = det(A_t) (1+ x_t^TA_{t}^{-1}x_t).
\end{split}
\end{equation*}
The third equation uses Sylvester's determinant theorem: $\det(I_m + A_{m\times n}B_{n\times m}) = \det(I_n + B_{n\times m}A_{m\times n})$. The trace of a matrix is the product of its eigenvalues and the determinant is the sum of eigenvalues, and for the trace of the positive definite matrix $A_t$ we have,
\begin{align*}
    \tr(A_t) = \tr(I + \sum_\tau^{t-1}x_\tau x^T_\tau) = d + \sum_\tau^{t-1}\norm{x_\tau}^2 \leq d + (t-1)\kappa^2.
\end{align*}
Therefore, using the inequality of arithmetic and geometric mean, $\det(A_t)\leq (1 + (t-1)\kappa^2/d)^{d}$.
\end{proof}

Summing over all the suboptimal bins, we have that
\begin{equation}\label{eq: meta bound on term one}
    \begin{split}
    &\sum_{k=1}^{n-1} \sum_{s=1}^{N^k(T)} 2(\sqrt{\beta_{k,s}}\lVert A_{k,s}^{-1/2}x_{k,s}\rVert\land 1) \leq \sum_{k=1}^n \sqrt{8d\beta_{k,N^k(T)} N^k(T)\ln\left(1 + N^k(T)\right)} \\
    &\leq \sqrt{\sum_{k=1}^{n}N^k(T) \sum_{k=1}^n 8d\beta_{k,N^k(T)}\ln\left(1 + N^k(T)\right)} \\
    &=\sqrt{T \sum_{k=1}^n 8d\beta_{k,N^k(T)}\ln\left(1 + N^k(T)\right)} \\
    &\stackrel{N^k(T)\leq T}{\leq} \sqrt{8d T n\beta_{T}\ln\left(1 + T\right)}.
    \end{split}
\end{equation}

\subsubsection{High probability regret bound part II (term $\#2$)}\label{sec: meta regret part II}
Here we directly call previous result in bound~(\ref{eq: bound on linear UCB term two}), but replace the total number of step with $N^k(T)$, the number of pulls for one fixed bin $k$. We have for term \#2, 
\[\sum_{s=1}^{N^k(T)}2\epsilon\sum_{\tau=1}^{s-1}\vert x_{k,s}^T A_{k,s}^{-1}x_{k,\tau}\vert \leq 2\epsilon N^k(T)d\sqrt{2\ln(1 + N^k(T))}.\]
Summing over all suboptimal bins, we have that
\begin{equation}
    \begin{split}
        &\sum_{k=1}^n 2\epsilon N^k(T)d\sqrt{2\ln(1 + N^k(T))}\\
        &\stackrel{N^k(T)\leq T}{\leq} 2\epsilon d \sqrt{2\ln(1+T)}\sum_{k=1}^n N^k(T)\\
        &=2\epsilon dT\sqrt{2\ln(1+T)}.
    \end{split}
\end{equation}

\subsubsection{Putting it together}
Combining the decomposition in equation~(\ref{eq: meta regret decomposition - one arm}) and the results in subsections~\ref{sec: meta regret part I} and~\ref{sec: meta regret part II}, we have a high probability regret bound for the UCB-Meta-algorithm:
\begin{equation}
    \begin{split}
    &R_T = \sum_{k=1}^n R_k \\
    &\leq \sqrt{8d T n\beta_{T}\ln\left(1 + T\right)} + 2\epsilon dT\sqrt{2\ln(1+T)} + 2\epsilon T \\
    & =  \bigo(d\ln(T)\sqrt{Tn\ln(T^2 n/\delta)} + \epsilon dT\sqrt{\ln(T)} + \epsilon T).
    \end{split}    
\end{equation}
The last step plugs in $\beta_T = \bigo( d\ln(T)\ln(T^2 n/\delta))$.

\end{proof}
\subsubsection{Proof of Theorem~\ref{thm: doubling for meta-algorithm}}\label{sec: proof of doubling thm}
    \begin{proof}
        Algorithm~\ref{alg: doubling trick for meta} executes Algorithm~\ref{alg: meta algorithm} for a sequence of pre-defined time periods, $\{T_i = 2^i, i = 0,1, \dots N\}$. At the beginning of each period, the update history is cleared and the number of arms $n$ is reset with respect to the current horizon $T_i$. 
        However, 
        since we would like to acquire a high-probability regret bound after applying the doubling trick, we need to set the fail probability of Meta-algorithms during period $i$ to $\delta_i = 6\delta/\pi^2 i^2$. 
        Using a union bound, we can conclude the following ($R_i(T_i)$ denotes the regret incurred in time period i of length $T_i$ only). 
        \begin{equation*}
            \begin{split}
                &\PP(\forall i, \text{the bound hold for } R_i(T_i)) \\
                & = 1 - \sum_i{\PP(\text{the bound does not hold for} R_i(T_i))} \\
                & = 1 - \sum_i \frac{6\delta}{\pi^2 i^2} 
                \approx 1 - \delta.
            \end{split}
        \end{equation*} 
        In the last step we use the fact that the sum of sequence $\sum_i^\infty \frac{1}{i^2}$ converges to $\frac{\pi^2}{6}$. 
         
        Now, the total regret is simply a summation over $i$. The following holds with probability $1-\delta$, 
        \begin{equation} \label{eq: doubling for meta-algorithm}
            \begin{split}
                R(T) &\leq \sum_{i=1}^N R_i(T_i) \\
                &\leq  \sum_{i=1}^N \tilde{\bigo}(d {T_i}^a) = \tilde{\bigo} \left(d \sum_{i=1}^N {2}^{ia}\right) \\
                &\leq \tilde{\bigo} \left( d 2^{a(N-1)}  \right) \\
                &=\tilde{\bigo} (d T^a) .
            \end{split}
        \end{equation}
        At step 4, the number of time periods $N$ is the smallest integer such that $\sum_{i=0}^N 2^i \geq T$, so $N = 1 + \ceil{\log_2(T)}$. The sum of geometric sequence is $2^{a\ceil{\log_2(T)}} = (2^{\log_2(T) + c})^a = T^a 2^{ca}$ for some constant $c$ smaller than 1.  
        Also, note that step 2 holds even though the fail probability is changed to $\delta_i = 6\delta/\pi^2 i$ is because as specified in Theorem~\ref{thm: meta main regret}, the term $\delta$ appears in a log term and the maximum value of $1/\delta$ is $1/\delta_N = \pi^2 \log_2(T) / 6\delta$, therefore the extra factor caused by smaller $\delta$ to the regret is still a log term of $T_i$ and omitted in the proof here. 
        
        Bound~(\ref{eq: doubling for meta-algorithm}) suffices to say that meta-algorithm with doubling trick has the same regret rate as meta-algorithm with known horizon, with some additional constant factors suffered from restarting. 
    \end{proof}

\subsection{Proof of Theorem~\ref{thm: regret of smoothcorral-meta}}\label{sec: proof of thm: regret of smoothcorral-meta}
\begin{proof}
Here we prove that Corral with smooth-wrapper is applicable to this task and achieves minimax expected regret rate apart from log factors. We directly use the proof of Theorem 5.3 in~\cite{pacchiano2020model} and their notations.
$\delta$ is the fail probability, $M$ is the number of base-algorithms, $\rho$ is the reciprocal of the smallest possibility for base-algorithms over the T rounds and $\eta$ is the learning rate. $U(T, \delta)$ is the high probability bound of the selected base-algorithm. The regret of Corral with smooth wrapper is bounded by:
\begin{equation}
    \begin{split}
        R(T) &\leq \bigo(\frac{M\ln(T)}{\eta}+ T\eta) +\delta T + 8\sqrt{MT\log(\frac{4TM}{\delta})}  - \EE[\frac{\rho}{40\eta\ln(T)} - 2\rho U(T/\rho,\delta)\log(T)],
    \end{split}
\end{equation}
and we know from Theorem~\ref{thm: main regret} in our paper that the base algorithm (Algorithm~\ref{alg: misspecified linear UCB}) that locates in the global maximum's bin has anytime high probability regret bound $U(T, \delta) = \tilde\bigo(\epsilon T \lineardim + c(\delta)\lineardim\sqrt{T})$, note that this is because the dimension of the local linear parameter is $\lineardim$. Therefore,
\begin{equation}
    \begin{split}
        R(T) 
        &= \tilde{\bigo}(\sqrt{MT} + \frac{M}{\eta}+T\eta) + \delta T - \EE[\frac{\rho}{40\eta\ln(T)} - 2\rho\tilde\bigo(\lineardim\sqrt{T/\rho}+\frac{\epsilon \lineardim T}{\rho})]) \\
        &=\tilde{\bigo}(\sqrt{MT} + \frac{M}{\eta}+T\eta) + \delta T + \tilde\bigo(\epsilon T \lineardim) + \EE[\tilde\bigo(\lineardim \sqrt{T\rho} - \frac{3\rho}{\eta})].
    \end{split}
\end{equation}
Firstly, we set $\delta = 1/T$ so that $\delta T = \bigo(1)$. Then we maximize this formulation over $\rho$ by setting $\rho = \tilde\bigo(\eta^2 {\lineardim}^2 T)$, yielding the following bound on expected regret.
\begin{equation}
    \begin{split}
    &\tilde{\bigo}(\sqrt{MT} + \frac{M}{\eta}+T\eta + \epsilon T {\lineardim} + \eta {\lineardim}^2 T)\\
    &\stackrel{\substack{M = n, \\ \epsilon = n^{-\frac{\alpha}{\xdim}}}}{=} \tilde{\bigo}(\sqrt{nT}+ \frac{n}{\eta} + n^{-\frac{\alpha}{\xdim}}T\lineardim + \eta {\lineardim}^2 T ).
    \end{split}
\end{equation}
We minimize this by setting the derivative w.r.t $n$ and $\eta$ to zero, i.e. $\eta = \frac{1}{\lineardim}\sqrt{\frac{n}{T}}$ and $n = \tilde\bigo(T^{\frac{\xdim}{\xdim+2\alpha}})$. As a result the rate comes to $\tilde\bigo(\lineardim T^{\frac{\xdim+\alpha}{\xdim+2\alpha}})$. 
\end{proof}

\subsection{Proof of Lemma~\ref{lemma: regret of meta with gamma}}\label{proof of lemma meta regret with input gamma}
\begin{proof}
    According to Theorem~\ref{thm: meta main regret}, the algorithm sets $n = T^{\frac{\xdim}{\xdim+2\alpha'}}/\ln(T)^{\frac{2\xdim}{\xdim+2\alpha'}}$ and $\epsilon = n^{\frac{-\alpha'}{\xdim}}$. 
    Note that we can only use the result in Theorem~\ref{thm: meta main regret} if the high probability upper confidence bound defined in line 4 of sub-procedure Algorithm~\ref{alg: meta algorithm} holds honestly.
    If the input parameter $\alpha'$ is larger than $\alpha$, then the calculated misspecification error $\epsilon$ is smaller than the true $\epsilon^*=\tilde{\bigo}(T^\frac{-\alpha}{\xdim+2\alpha})$, causing the confidence bound to be invalid. Therefore, the regret bound does not hold for when $\alpha'>\alpha$.
    If the input parameter is smaller than $\alpha$, then we can simply use the fact that functions that are $\alpha$-H\"older smooth are also $\alpha'$-H\"older smooth: $H(\alpha, L)\subset H(\alpha', L)$. Therefore, the regret of the algorithm with input parameter $\alpha'\leq \alpha$ is bounded by
        $R(T)\leq \tilde\bigo(d(\alpha')(\sqrt{Tn} + \epsilon T)) = \tilde\bigo(d(\alpha') T^{\frac{\xdim+\alpha'}{\xdim+2\alpha'}})$.
    \end{proof}

\subsection{Proof of Theorem~\ref{thm: adaptation, smooth corral}}\label{sec: proof of thm: adaptation, smooth corral}
\begin{proof}
    There exists an $\estalpha\in \mathcal{G}$, s.t.  $\estalpha \leq \alpha \leq \estalpha + \frac{R}{\log(T)}$, for any true $\alpha$ in $(0, R]$.
    There are two sources that made up the cost of adaptation when using Corral. The first one is the cost of searching over a grid for the unknown point $\hat\alpha$. The second one is the cost of approximation, specifically the difference between the rates achieved for $\hat\alpha$ and the true $\alpha$. We will first derive the cost of grid search.
   
    As specified in the proof of Theorem 5.3 in~\cite{pacchiano2020model}, the following bound of regret of the Corral algorithm holds with respect to any of its base-algorithm with high probability regret bound $U(T,\delta)$. The notations were introduced in Appendix section~\ref{sec: proof of thm: regret of smoothcorral-meta}.
    \begin{equation}
        R(T) \leq \bigo(\frac{M\ln(T)}{\eta} + T\eta) - \EE[\frac{\rho}{40\eta\ln(T)} - 2\rho U(T/\rho,\delta)\log(T)]+\delta T + 8\sqrt{MT\log(\frac{4TM}{\delta})}.
    \end{equation}
    Plugging the regret rate of base-algorithm in Lemma~\ref{lemma: regret of meta with gamma}, the expected pseudo-regret of Corral with smooth wrapper is therefore bounded by:
    \begin{equation}
        \begin{split}
            R(T)&\stackrel{\estalpha \leq \alpha}{\leq} \tilde\bigo(\frac{M}{\eta} + T\eta + \sqrt{MT}) + \delta T - \EE[\frac{\rho}{40\eta\ln(T)} - 2\rho(\tilde{\bigo}(d(\frac{T}{\rho})^{\frac{d+\estalpha}{d+2\estalpha}}))\log(T)] \\
            & \stackrel{\delta = 1/T}{=} \tilde\bigo(\frac{M}{\eta} + T\eta + \sqrt{MT}) - \EE[\tilde\bigo(\frac{\rho}{\eta} - \rho d(\frac{T}{\rho})^{\frac{d+\estalpha}{d+2\estalpha}})] \\
            & = \tilde\bigo(\frac{M}{\eta} + T\eta + \sqrt{MT}) - \EE[\tilde\bigo(\frac{\rho}{\eta} - d T^{\frac{d+\estalpha}{d+2\estalpha}} \rho^{\frac{\estalpha}{d+2\estalpha}})]. 
        \end{split}
    \end{equation}
   
    Similarly, we first maximize over $\rho$ by setting the derivative w.r.t $\rho$ to zero by setting $\rho = \tilde\bigo(\eta^{\frac{d+2\estalpha}{d+\estalpha}} d^\frac{d+2\estalpha}{d+\estalpha} T)$. Then the above rate comes to
    \begin{equation}
        R(T)\leq \tilde{\bigo}(\frac{M}{\eta} + T\eta + \sqrt{MT} + d^\frac{d+2\estalpha}{d+\estalpha} T \eta^{\frac{\estalpha}{d+\estalpha}}). 
    \end{equation}
   
    However, since $\eta$ is a parameter of the Corral algorithm which does not know $\estalpha$ or $\alpha$, we will rely on the parameter $R$ specified by the user. Let us set $\eta$ with repsect to $\alpha = R$, i.e. $\eta = \tilde{\bigo}(d^{-1} T^{-\frac{d+R}{d+2R}})$, and plug in the number of grid points (base-algorithms) $M = \lceil \log(T)\rceil$.
    \begin{equation}\label{eq: rate of corral_smooth + meta grid} 
        \begin{split}
            &\tilde{\bigo}(\frac{M}{\eta} + T\eta + \sqrt{MT} + d^\frac{d+2\estalpha}{d+\estalpha} T \eta^{\frac{\estalpha}{d+\estalpha}}) \\
            &=\tilde{\bigo}(dT^{\frac{d+R}{d+2R}} + d^{-1}T^{\frac{R}{d+2R}} + dT^\frac{d^2+2Rd + R\estalpha}{(d+2R)(d+\estalpha)}) \\
            &=\tilde{\bigo}(dT^{\frac{d+R}{d+2R}} + dT^\frac{d^2+2Rd + R\estalpha}{(d+2R)(d+\estalpha)}).
        \end{split}
    \end{equation} 
    It is obvious that this rate is not the minimax optimal rate for class $\sum(\hat\alpha)$, this gap shows the cost of grid search. 
    
    Next, let us consider the cost of approximation and how it is eliminated by using the linear grid~(\cite{hoffmann2011adaptive}). Namely, we show that adaptation for $\hat\alpha$ is equivalent to adaptation for $\alpha$:
    \begin{equation}\label{eq: rate of corral_smooth + meta} 
        \tilde{\bigo}(dT^{\frac{d+R}{d+2R}} + dT^\frac{d^2+2Rd + R\estalpha}{(d+2R)(d+\estalpha)}) = \tilde\bigo(dT^{\frac{d+R}{d+2R}} + dT^\frac{d^2+2Rd + R\alpha}{(d+2R)(d+\alpha)}).
    \end{equation}
    The equality holds because $\vert \alpha - \estalpha\vert\leq \frac{R}{\log(T)}$. Let $J = \frac{d^2+2Rd + R\alpha}{(d+2R)(d+\estalpha)}$ and $Q = \frac{d^2+2Rd + R\alpha}{(d+2R)(d+\alpha)}$, then $W \stackrel{\triangle}{=}\frac{T^J}{T^Q} \leq T^\frac{(d^2 + 2Rd + R\alpha)\frac{R}{\log(T)}}{(d+2R)(d+\alpha)(d+\estalpha)}$. Taking the log of $W$ yields $\log(W) = R \frac{d^2 + 2Rd + R\alpha}{(d+2R)(d+\alpha)(d+\estalpha)}$. Since both $\alpha$ and $\estalpha$ are bounded by a constant range $(0, 2]$, the term $\frac{d^2 + 2Rd + R\alpha}{(d+2R)(d+\alpha)(d+\estalpha)} \leq C$ for some constant $C$, $W$ is therefore $\bigo(1)$ as well. 

    Therefore, for functions with H\"older exponent $\alpha<R$, the second term in equation~(\ref{eq: rate of corral_smooth + meta}) is the dominant term and the expected regret rate is $\tilde\bigo(dT^\frac{d^2+2Rd + R\alpha}{(d+2R)(d+\alpha)})$. For functions with H\"older exponent $\alpha\geq R$, which essentially belongs to a subset of $\sum(R, L)$, they will all have the same rate which is $\tilde{\bigo}(dT^{\frac{d+R}{d+2R}})$. When $\alpha = R$, this matches the minimax rate for $\alpha$. 
\end{proof}
\newpage
\section{Additional algorithms for the main document}
\subsection{Doubling procedure for Algorithm~\ref{alg: meta algorithm}}\label{sec: doubling for meta algorithm}
\begin{algorithm}[ht]
    \caption{Doubling procedure for Algorithm~\ref{alg: meta algorithm}}
    \begin{algorithmic}[1]\label{alg: doubling trick for meta}
        \REQUIRE Meta-algorithm $\Acal^{global}$ (Algorithm~\ref{alg: meta algorithm}), fail probability $\delta$
        \FOR{$i = 0 \dots $}
            \STATE $T_i = 2^i$
            \STATE Restart $\Acal^{global}$ with initialization parameters $n_i = \lfloor T_i^{\frac{d}{d+2\alpha}}/\ln(T_i)^{\frac{2d}{d+2\alpha}}\rceil$
            and fail probability $\delta_i = 6\delta/\pi^2 i^2$
            \STATE Run $\Acal^{global}$ for $T_i$ steps.
        \ENDFOR
    \end{algorithmic}
\end{algorithm}

\subsection{The Corral Master algorithm}\label{sec: corral algo reference}
For easier reference, we include the copy of Algorithm 7 in~\cite{pacchiano2020model}. 
\begin{algorithm}[ht]
    \caption{Corral Master (Algorithm 7 in~\cite{pacchiano2020model})}
    \begin{algorithmic}[1]\label{alg: corral master}
        \REQUIRE Base algorithms $\{\mathcal{B}_j\}_{j=1}^M$, learning rate $\eta$.
        \STATE Initialize: $\gamma = 1/T, \beta = e^{\frac{1}{\ln(T)}}, \eta_{1,j} = \eta, \rho^j_1 = 2M, \underline{p}^j_1 = \frac{1}{\rho_{1}^j}, p_1^j = 1/M$ for all $j\in [M]$. 
        \FOR{$t=1, \dots , T$}
            \STATE Sample $i_t\sim p_t$. 
            \STATE Receive feedback $r_t$ from base $\mathcal{B}_{i_t}$. 
            \STATE Update $p_t, \eta_t$ and $\underline p_{t}$ to $p_{t+1}, \eta_{t+1}$ and $\underline p_{t+1}$using $r_t$ via Corral-Update (takes input $\eta_t, p_t,\beta$, lower bound $\underline p_t$ and current feedback $r_t$).
            \FOR{j=1, \dots,  M}
            \STATE Set $\rho_{t+1}^j = \frac{1}{\underline p_{t+1}^j}$.  
            \ENDFOR
        \ENDFOR
    \end{algorithmic}
\end{algorithm}

The corral update procedure is in Algorithm 5 and the smooth wrapper for the base-algorithms in Algorithm 3 in~\cite{pacchiano2020model}.


\begin{thebibliography}{29}
\providecommand{\natexlab}[1]{#1}
\providecommand{\url}[1]{\texttt{#1}}
\expandafter\ifx\csname urlstyle\endcsname\relax
  \providecommand{\doi}[1]{doi: #1}\else
  \providecommand{\doi}{doi: \begingroup \urlstyle{rm}\Url}\fi

\bibitem[Abbasi-Yadkori et~al.(2011)Abbasi-Yadkori, P{\'a}l, and
  Szepesv{\'a}ri]{abbasi2011improved}
Yasin Abbasi-Yadkori, D{\'a}vid P{\'a}l, and Csaba Szepesv{\'a}ri.
\newblock Improved algorithms for linear stochastic bandits.
\newblock In \emph{Advances in Neural Information Processing Systems}, pages
  2312--2320, 2011.

\bibitem[Agarwal et~al.(2016)Agarwal, Luo, Neyshabur, and
  Schapire]{agarwal2016corralling}
Alekh Agarwal, Haipeng Luo, Behnam Neyshabur, and Robert~E Schapire.
\newblock Corralling a band of bandit algorithms.
\newblock \emph{arXiv preprint arXiv:1612.06246}, 2016.

\bibitem[Akhavan et~al.(2020)Akhavan, Pontil, and
  Tsybakov]{akhavan2020exploiting}
Arya Akhavan, Massimiliano Pontil, and Alexandre~B Tsybakov.
\newblock Exploiting higher order smoothness in derivative-free optimization
  and continuous bandits.
\newblock \emph{arXiv preprint arXiv:2006.07862}, 2020.

\bibitem[Arora et~al.(2020)Arora, Marinov, and Mohri]{arora2020corralling}
Raman Arora, Teodor~V Marinov, and Mehryar Mohri.
\newblock Corralling stochastic bandit algorithms.
\newblock \emph{arXiv preprint arXiv:2006.09255}, 2020.

\bibitem[Auer et~al.(1995)Auer, Cesa-Bianchi, Freund, and
  Schapire]{auer1995gambling}
Peter Auer, Nicolo Cesa-Bianchi, Yoav Freund, and Robert~E Schapire.
\newblock Gambling in a rigged casino: The adversarial multi-armed bandit
  problem.
\newblock In \emph{Proceedings of IEEE 36th Annual Foundations of Computer
  Science}, pages 322--331. IEEE, 1995.

\bibitem[Auer et~al.(2002)Auer, Cesa-Bianchi, Freund, and
  Schapire]{auer2002nonstochastic}
Peter Auer, Nicolo Cesa-Bianchi, Yoav Freund, and Robert~E Schapire.
\newblock The nonstochastic multiarmed bandit problem.
\newblock \emph{SIAM journal on computing}, 32\penalty0 (1):\penalty0 48--77,
  2002.

\bibitem[Auer et~al.(2007)Auer, Ortner, and Szepesv{\'a}ri]{auer2007improved}
Peter Auer, Ronald Ortner, and Csaba Szepesv{\'a}ri.
\newblock Improved rates for the stochastic continuum-armed bandit problem.
\newblock In \emph{International Conference on Computational Learning Theory},
  pages 454--468. Springer, 2007.

\bibitem[Bubeck et~al.(2010)Bubeck, Munos, Stoltz, and Szepesvari]{bubeck2010x}
S{\'e}bastien Bubeck, R{\'e}mi Munos, Gilles Stoltz, and Csaba Szepesvari.
\newblock X-armed bandits.
\newblock \emph{arXiv preprint arXiv:1001.4475}, 2010.

\bibitem[Bubeck et~al.(2011)Bubeck, Stoltz, and Yu]{bubeck2011lipschitz}
S{\'e}bastien Bubeck, Gilles Stoltz, and Jia~Yuan Yu.
\newblock Lipschitz bandits without the lipschitz constant.
\newblock In \emph{International Conference on Algorithmic Learning Theory},
  pages 144--158. Springer, 2011.

\bibitem[Dani et~al.(2008)Dani, Hayes, and Kakade]{dani2008stochastic}
Varsha Dani, Thomas~P Hayes, and Sham~M Kakade.
\newblock Stochastic linear optimization under bandit feedback.
\newblock 2008.

\bibitem[Foster et~al.(2016)Foster, Li, Lykouris, Sridharan, and
  Tardos]{foster2016learning}
Dylan~J Foster, Zhiyuan Li, Thodoris Lykouris, Karthik Sridharan, and Eva
  Tardos.
\newblock Learning in games: Robustness of fast convergence.
\newblock In \emph{Advances in Neural Information Processing Systems}, pages
  4734--4742, 2016.

\bibitem[Foster et~al.(2019)Foster, Krishnamurthy, and Luo]{foster2019model}
Dylan~J Foster, Akshay Krishnamurthy, and Haipeng Luo.
\newblock Model selection for contextual bandits.
\newblock In \emph{Advances in Neural Information Processing Systems}, pages
  14741--14752, 2019.

\bibitem[Grant and Leslie(2020)]{grant2020thompson}
James~A Grant and David~S Leslie.
\newblock On thompson sampling for smoother-than-lipschitz bandits.
\newblock \emph{arXiv preprint arXiv:2001.02323}, 2020.

\bibitem[Gur et~al.(2019)Gur, Momeni, and Wager]{gur2019smoothness}
Yonatan Gur, Ahmadreza Momeni, and Stefan Wager.
\newblock Smoothness-adaptive stochastic bandits.
\newblock \emph{arXiv preprint arXiv:1910.09714}, 2019.

\bibitem[Hoffmann et~al.(2011)Hoffmann, Nickl, et~al.]{hoffmann2011adaptive}
Marc Hoffmann, Richard Nickl, et~al.
\newblock On adaptive inference and confidence bands.
\newblock \emph{The Annals of Statistics}, 39\penalty0 (5):\penalty0
  2383--2409, 2011.

\bibitem[Hu et~al.(2020)Hu, Kallus, and Mao]{hu2020smooth}
Yichun Hu, Nathan Kallus, and Xiaojie Mao.
\newblock Smooth contextual bandits: Bridging the parametric and
  non-differentiable regret regimes.
\newblock In \emph{Conference on Learning Theory}, pages 2007--2010, 2020.

\bibitem[Kleinberg et~al.(2008)Kleinberg, Slivkins, and
  Upfal]{kleinberg2008multi}
Robert Kleinberg, Aleksandrs Slivkins, and Eli Upfal.
\newblock Multi-armed bandits in metric spaces.
\newblock In \emph{Proceedings of the fortieth annual ACM symposium on Theory
  of computing}, pages 681--690, 2008.

\bibitem[Kleinberg(2005)]{kleinberg2005nearly}
Robert~D Kleinberg.
\newblock Nearly tight bounds for the continuum-armed bandit problem.
\newblock In \emph{Advances in Neural Information Processing Systems}, pages
  697--704, 2005.

\bibitem[Krishnamurthy et~al.(2019)Krishnamurthy, Langford, Slivkins, and
  Zhang]{krishnamurthy2019contextual}
Akshay Krishnamurthy, John Langford, Aleksandrs Slivkins, and Chicheng Zhang.
\newblock Contextual bandits with continuous actions: Smoothing, zooming, and
  adapting.
\newblock \emph{arXiv preprint arXiv:1902.01520}, 2019.

\bibitem[Lattimore and Szepesvari(2019)]{lattimore2019learning}
Tor Lattimore and Csaba Szepesvari.
\newblock Learning with good feature representations in bandits and in rl with
  a generative model.
\newblock \emph{arXiv preprint arXiv:1911.07676}, 2019.

\bibitem[Locatelli and Carpentier(2018)]{locatelli2018adaptivity}
Andrea Locatelli and Alexandra Carpentier.
\newblock Adaptivity to smoothness in x-armed bandits.
\newblock In \emph{Conference on Learning Theory}, pages 1463--1492, 2018.

\bibitem[Low et~al.(1997)]{low1997nonparametric}
Mark~G Low et~al.
\newblock On nonparametric confidence intervals.
\newblock \emph{The Annals of Statistics}, 25\penalty0 (6):\penalty0
  2547--2554, 1997.

\bibitem[Pacchiano et~al.(2020)Pacchiano, Phan, Abbasi-Yadkori, Rao, Zimmert,
  Lattimore, and Szepesvari]{pacchiano2020model}
Aldo Pacchiano, My~Phan, Yasin Abbasi-Yadkori, Anup Rao, Julian Zimmert, Tor
  Lattimore, and Csaba Szepesvari.
\newblock Model selection in contextual stochastic bandit problems.
\newblock \emph{arXiv preprint arXiv:2003.01704}, 2020.

\bibitem[Rusmevichientong and Tsitsiklis(2010)]{rusmevichientong2010linearly}
Paat Rusmevichientong and John~N Tsitsiklis.
\newblock Linearly parameterized bandits.
\newblock \emph{Mathematics of Operations Research}, 35\penalty0 (2):\penalty0
  395--411, 2010.

\bibitem[Stone(1982)]{stone1982optimal}
Charles~J Stone.
\newblock Optimal global rates of convergence for nonparametric regression.
\newblock \emph{The annals of statistics}, pages 1040--1053, 1982.

\bibitem[Tsybakov(2008)]{tsybakov2008introduction}
Alexandre~B Tsybakov.
\newblock \emph{Introduction to nonparametric estimation}.
\newblock Springer Science \& Business Media, 2008.

\bibitem[Wainwright(2019)]{wainwright2019high}
Martin~J Wainwright.
\newblock \emph{High-dimensional statistics: A non-asymptotic viewpoint},
  volume~48.
\newblock Cambridge University Press, 2019.

\bibitem[Wang et~al.(2018)Wang, Balakrishnan, and Singh]{wang2018optimization}
Yining Wang, Sivaraman Balakrishnan, and Aarti Singh.
\newblock Optimization of smooth functions with noisy observations: Local
  minimax rates.
\newblock In \emph{Advances in Neural Information Processing Systems}, pages
  4338--4349, 2018.

\bibitem[Yatchew(1998)]{yatchew1998nonparametric}
Adonis Yatchew.
\newblock Nonparametric regression techniques in economics.
\newblock \emph{Journal of Economic Literature}, 36\penalty0 (2):\penalty0
  669--721, 1998.

\end{thebibliography}
\end{document}